\documentclass{article}

\newif\ifpreprint
\preprintfalse

     \usepackage[final]{neurips_2020}

\usepackage[utf8]{inputenc} 
\usepackage[T1]{fontenc}    
\usepackage{hyperref}       
\usepackage{url}            
\usepackage{booktabs}       
\usepackage{amsfonts}       
\usepackage{nicefrac}       
\usepackage{microtype}      
\usepackage{mathtools}
\usepackage{wrapfig}

\usepackage{tikz}
\usetikzlibrary{arrows.meta}

\usepackage{amssymb}

\usepackage{xifthen}
\usepackage{xparse}
\usepackage{dsfont}
\usepackage{xspace}

\usepackage{xcolor}
\usepackage{lscape}

\usepackage{subcaption}

\usepackage{amsthm}
\newtheorem{theorem}{Theorem}
\newtheorem{lemma}[theorem]{Lemma}

\newcommand{\lr}[1]{\left (#1\right)}
\newcommand{\lrs}[1]{\left [#1 \right]}
\newcommand{\lrc}[1]{\left \{#1 \right\}}

\newcommand{\V}{\mathbb{V}}
\newcommand{\D}{\mathbb{D}}
\let\L\undefined
\newcommand{\L}{L}

\DeclareMathOperator{\testset}{S_{test}}

\NewDocumentCommand{\1}{o}{\mathds 1{\IfValueT{#1}{\lr{#1}}}}
\NewDocumentCommand{\E}{o}{\mathbb E\IfValueT{#1}{\lrs{#1}}}
\NewDocumentCommand{\Var}{o}{\V\IfValueT{#1}{\lrs{#1}}}
\let\P\undefined
\NewDocumentCommand{\P}{o}{\mathbb P{\IfValueT{#1}{\lr{#1}}}}

\DeclareMathOperator{\MV}{MV}

\DeclareMathOperator{\KL}{KL}
\DeclareMathOperator{\kl}{kl}

\newcommand{\dataset}[1]{\textsc{#1}}

\newcommand{\numhyp}{M}
\newcommand{\hatLtnd}{\hat L_{\texttt{tnd}}}

\newcommand{\nmin}{n_{\texttt{min}}}

\newcommand{\FO}{\operatorname{FO}}
\newcommand{\TND}{\operatorname{TND}}
\newcommand{\DIS}{\operatorname{DIS}}
\newcommand{\Cone}{\operatorname{C1}}
\newcommand{\Ctwo}{\operatorname{C2}}
\newcommand{\CTD}{\operatorname{CTD}}

\newcommand{\myparagraph}[1]{\textbf{#1}\quad}

\title{Second Order PAC-Bayesian Bounds\\for the Weighted Majority Vote}

\author{
  Andr{\'e}s R. Masegosa\thanks{Part of the work was done while AM was visiting the University of Copenhagen.}
  \\
  University of Almer{\' i}a\\
  \texttt{andresma@ual.es}
  \And
  Stephan S. Lorenzen ~~~~~~
  Christian Igel ~~~~~~
  Yevgeny Seldin\\
  University of Copenhagen\\
  \texttt{\{lorenzen,igel,seldin\}@di.ku.dk}
}

\begin{document}

\maketitle

\begin{abstract}
We present a novel analysis of the expected risk of weighted majority vote in multiclass classification. The analysis takes correlation of predictions by ensemble members into account and provides a bound that is amenable to efficient minimization, which yields improved weighting for the majority vote. 
We also provide a specialized version of our bound for binary classification, which allows to exploit additional unlabeled data for tighter risk estimation. In experiments, we apply the bound to improve weighting of trees in random forests and show that, in contrast to the commonly used first order bound, minimization of the new bound typically does not lead to degradation of the test error of the ensemble.
\end{abstract}

\section{Introduction}
Weighted majority vote is a fundamental technique for combining predictions of multiple classifiers. 
In machine learning, it was proposed for neural networks by  \cite{hansen:90}
and became popular 
with the works of \citet{Bre96,Bre01} on bagging and random forests and the work of \citet{FS96} on boosting. \citet{Zhu15} surveys the subsequent development of the field. Weighted majority vote is now part of the winning strategies in many machine learning competitions \citep[e.g.,][]{chen2016xgboost,hoch2015ensemble,puurula2014kaggle,stallkamp:12}. Its power lies in the cancellation of errors effect \citep{eckhardt:85}: when individual classifiers perform better than a random guess and make independent errors, the errors average out and the majority vote tends to outperform the individual classifiers.

A central question in the design of a weighted majority vote is the assignment of weights to individual classifiers. This question was resolved by \citet{BK16}  under the assumptions that the expected error rates of the classifiers are known and their errors are independent. However, neither of the two assumptions is typically satisfied in practice.

When the expected error rates are estimated based on a sample, the common way of bounding the expected error of a weighted majority vote is by twice the error of the corresponding randomized classifier \citep{LST02}. A randomized classifier, a.k.a.\ Gibbs classifier, associated with a distribution (weights) $\rho$ over classifiers draws a single classifier at random at each prediction round according to $\rho$ and applies it to make the prediction. The error rate of the randomized classifier is bounded using PAC-Bayesian analysis \citep{McA98,See02,LST02}. We call this a \emph{first order bound}. The factor 2 bound on the gap between the error of the weighted majority vote and the corresponding randomized classifier follows from the observation that an error by the weighted majority vote implies an error by at least a weighted half of the base classifiers. The bound is derived using Markov's inequality. While the PAC-Bayesian bounds for the randomized classifier are remarkably tight \citep{GLLM09,TIWS17}, the factor 2 gap is only tight in the worst-case, but loose in most real-life situations, where the weighted majority vote typically performs better than the randomized classifier rather than twice worse. The reason for looseness is that the approach does not take the correlation of errors into account.

In order to address the weakness of the first order bound, \citet{LLM+07} have proposed PAC-Bayesian C-bounds, which are based on Chebyshev-Cantelli inequality (a.k.a.\ one-sided Chebyshev's inequality) and take correlations into account. The idea was further developed by \citet{LMR11}, \citet{GLL+15}, and \citet{LMRR17}. However, the C-bounds have two severe limitations: (1) They are defined in terms of classification margin and the second moment of the margin is in the denominator of the bound. The second moment is difficult to estimate from data and significantly weakens the tightness of the bounds \citep{LIS19}. (2) The C-bounds are difficult to optimize. \citet{GLL+15} were only able to minimize the bounds in a highly restrictive case of self-complemented sets of voters and aligned priors and posteriors. In binary classification a set of voters is self-complemented if for any hypothesis $h\in\cal{H}$ the mirror hypothesis $-h$, which always predicts the opposite label to the one predicted by $h$, is also in $\cal{H}$. A posterior $\rho$ is aligned on a prior $\pi$ if $\rho(h) + \rho(-h) = \pi(h) + \pi(-h)$ for all $h\in\cal{H}$. Obviously, not every hypothesis space is self-complemented and such sets can only be defined in binary, but not in multiclass classification. Furthermore, the alignment requirement only allows to shift the posterior mass within the mirror pairs $(h,-h)$, but not across pairs. If both $h$ and $-h$ are poor classifiers and their joint prior mass is high, there is no way to remedy this in the posterior.

\citet{LIS19} have shown that for standard random forests applied to several UCI datasets the first order bound is typically tighter than the various forms of C-bounds proposed by \citet{GLL+15}. However, the first order approach has its own limitations. While it is possible to minimize the bound 
\citep{TIWS17}, it 
ignores the correlation of errors and minimization of the bound concentrates the weight on 
a few top classifiers and reduces the power of the ensemble. Our experiments show that minimization of 
the first order bound typically leads to deterioration of the test error.

We propose a novel analysis of the risk of weighted majority vote in multiclass classification, which addresses the weaknesses of previous methods. The new analysis is based on a second order Markov's inequality, $\P[Z \geq \varepsilon] \leq \E[Z^2]/\varepsilon^2$, which can be seen as a relaxation of the Chebyshev-Cantelli inequality. We use the inequality to bound the expected loss of weighted majority vote by four times the expected \emph{tandem loss} of the corresponding randomized classifier: The tandem loss measures the probability that two hypotheses drawn independently by the randomized classifier simultaneously err on a sample. Hence, it takes correlation of errors into account. We then use PAC-Bayesian analysis to bound the expected tandem loss in terms of its empirical counterpart and provide a procedure for minimizing the bound and optimizing the weighting. We show that the bound is reasonably tight and that, in contrast to the first order bound, minimization of the bound typically does not deteriorate the performance of the majority vote on new data. 

We also present a specialized version of the bound  for binary classification, which takes advantage of unlabeled data. It expresses the expected tandem loss in terms of a difference between the expected loss and half the expected disagreement between pairs of hypotheses. In the binary case the disagreements do not depend on the labels and can be estimated from unlabeled data, whereas the loss of a randomized classifier is a first order quantity, which is easier to estimate than the tandem loss. We note, however, that the specialized version only gives advantage over the general one when the amount of unlabeled data is considerably larger than the amount of labeled data.

\section{General problem setup}
\label{sec:generalsetup}

\myparagraph{Multiclass classification}
Let $S=\{(X_1,Y_1),\ldots,(X_n,Y_n)\}$ be an independent identically
distributed sample from $\cal{X}\times\cal{Y}$, drawn according to an unknown distribution $D$, where $\cal{Y}$ is finite and $\cal{X}$ is arbitrary. A hypothesis
is a function $h : \cal{X} \rightarrow \cal{Y}$, and $\cal H$ denotes a space of hypotheses. We evaluate the quality of a hypothesis
$h$  by the 0-1 loss $\ell(h(X),Y)=\1[h(X)\neq Y]$, where $\1[\cdot]$ is the indicator function. The expected loss of $h$ is denoted by $L(h) =
\mathbb{E}_{(X,Y)\sim D} [\ell(h(X),Y)]$ and the empirical loss of $h$ on a sample $S$ of size $n$ is denoted by $\hat{L}(h,S) =
\frac{1}{n} \sum_{i=1}^n \ell (h(X_i),Y_i)$. 

\myparagraph{Randomized classifiers}
A \emph{randomized classifier} (a.k.a.\ Gibbs classifier) associated with a distribution $\rho$ on $\cal{H}$, for each input $X$ randomly draws a hypothesis $h\in{\cal H}$ according to $\rho$ and predicts $h(X)$. The expected loss of a randomized classifier is given by $\mathbb{E}_{h\sim \rho} [L(h)]$ and the empirical loss by $\mathbb{E}_{h\sim\rho}[\hat{L}(h,S)]$. 
To simplify the notation we use $\E_D[\cdot]$ as a shorthand for $\E_{(X,Y)\sim D}[\cdot]$ and $\E_\rho[\cdot]$ as a shorthand for $\E_{h\sim \rho}[\cdot]$. 

\myparagraph{Ensemble classifiers and majority vote}
Ensemble classifiers predict by taking a weighted aggregation of predictions by hypotheses from ${\cal H}$. The $\rho$-weighted majority vote $\MV_\rho$ predicts $\MV_\rho (X)= \arg\max_{y\in{\cal Y}} \E_\rho[\1[h(X) = y]]
$, where ties can be resolved arbitrarily. 

If majority vote makes an error, we know that at least a $\rho$-weighted half of the classifiers have made an error and, therefore, $\ell(\MV_\rho(X),Y) \leq \1[\E_\rho[\1[h(X)\neq Y]] \geq 0.5]$. This observation leads to the well-known first order oracle bound for the loss of weighted majority vote.
\begin{theorem}[First Order Oracle Bound]
\label{thm:first-order}
\[
L(\MV_\rho)\leq 2\E_\rho[L(h)].
\]
\end{theorem}
\begin{proof}
We have $L(\MV_\rho) = \E_D[\ell(\MV_\rho(X),Y)] \leq \P[\E_\rho[\1[h(X)\neq Y]] \geq 0.5]$.  By applying Markov's inequality to random variable $Z = \E_\rho[\1[h(X)\neq Y]]$ we have: 
\begin{equation*}
L(\MV_\rho) \leq \P[\E_\rho[\1[h(X)\neq Y]] \geq 0.5] \leq 2\E_D[\E_\rho[\1[h(X)\neq Y]]] = 2\E_\rho[L(h)].
\qedhere
\end{equation*}
\end{proof}
PAC-Bayesian analysis can be used to bound $\E_\rho[L(h)]$ in Theorem~\ref{thm:first-order} in terms of $\E_\rho[\hat L(h,S)]$, thus turning the oracle bound into an empirical one. The disadvantage of the first order approach is that $\E_\rho[L(h)]$ ignores correlations of predictions, which is the main power of the majority vote.

\section{New second order oracle bounds for the majority vote}

The key novelty of our approach is using a second order Markov's inequality: for a non-negative random variable $Z$ and $\varepsilon > 0$, we have $\P[Z \geq \varepsilon] = \P[Z^2 \geq \varepsilon^2] \leq \varepsilon^{-2}\E[Z^2]$. We define the \emph{tandem loss} of two hypotheses $h$ and $h'$ on a sample $(X,Y)$ by $\ell(h(X),h'(X),Y) = \1[h(X)\neq Y \wedge h'(X)\neq Y]$. (\citet{LLM+07} and \citet{GLL+15} use the term joint error for this quantity.) The tandem loss counts an error on a sample $(X,Y)$ only if both $h$ and $h'$ err on it. The \emph{expected tandem loss} is defined by
\[
\L(h,h') = \E_D[\1[h(X)\neq Y \wedge h'(X)\neq Y]].
\]
The following lemma, given as equation (7) by \citet{LLM+07} without a proof, relates the expectation of the second moment of the standard loss to the expected tandem loss. We use $\rho^2$ as a shorthand for the product distribution $\rho \times \rho$ over ${\cal H} \times {\cal H}$ and the shorthand $\E_{\rho^2}[L(h,h')] = \E_{h\sim\rho, h'\sim\rho}[L(h,h')]$.
\begin{lemma}In multiclass classification
\label{lem:second-monent}
\[
\E_D[\E_\rho[\1[h(X) \neq Y]]^2] = \E_{\rho^2}[\L(h,h')].
\]
\end{lemma}
A proof is provided in Appendix~\ref{app:L2D}. A combination of second order Markov's inequality with Lemma~\ref{lem:second-monent} leads to the following result.
\begin{theorem}[Second Order Oracle Bound]
\label{thm:MV-bound}
In multiclass classification
\begin{equation}
\label{eq:MV-bound}
L(\MV_\rho) \leq 4\E_{\rho^2}[\L(h,h')].
\end{equation}
\end{theorem}
\begin{proof}
By second order Markov's inequality applied to $Z = \E_\rho[\1[h(X)\neq Y]]$ and Lemma~\ref{lem:second-monent}:
\begin{equation*}
L(\MV_\rho) \leq \P[\E_\rho[\1[h(X) \neq Y]] \geq 0.5] \leq 4\E_D[\E_\rho[\1[h(X)\neq Y]]^2] = 4\E_{\rho^2}[\L(h,h')].
\qedhere
\end{equation*}
\end{proof}

\subsection{A specialized bound for binary classification}
We provide an alternative form of Theorem~\ref{thm:MV-bound}, which can be used to exploit unlabeled data in binary classification. We denote the \emph{expected disagreement} between hypotheses $h$ and $h'$ by $\D(h,h') = \E_D[\1[h(X)\neq h'(X)]]$ and express the tandem loss in terms of standard loss and disagreement. (The lemma is given as equation (8) by \citet{LLM+07} without a proof.)
\begin{lemma}
\label{lem:L2D}
In binary classification
\[
\E_{\rho^2}[L(h,h')] = \E_\rho[L(h)] - \frac{1}{2}\E_{\rho^2}[\D(h,h')].
\]
\end{lemma}
A proof of the lemma is provided in Appendix~\ref{app:L2D}. The lemma leads to the following result.
\begin{theorem}[Second Order Oracle Bound for Binary Classification]
\label{thm:MV-bound-binary}
In binary classification
\begin{equation}
\label{eq:MV-binary}
L(\MV_\rho) \leq 4\E_\rho[L(h)] - 2\E_{\rho^2}[\D(h,h')].
\end{equation}
\end{theorem}
\begin{proof}
The theorem follows by plugging the result of Lemma~\ref{lem:L2D} into Theorem~\ref{thm:MV-bound}.
\end{proof}
The advantage of the alternative way of writing the bound is the possibility of using unlabeled data for estimation of $\D(h,h')$ in binary prediction (see also \citealp{GLL+15}). We note, however, that estimation of $\E_{\rho^2}[\D(h,h')]$ has a slow convergence rate, as opposed to $\E_{\rho^2}[L(h,h')]$, which has a fast convergence rate. We discuss this point in Section~\ref{sec:fast-vs-slow}.

\subsection{Comparison with the first order oracle bound}
\label{sec:FOvsSO}

From Theorems~\ref{thm:first-order} and \ref{thm:MV-bound-binary} we see that in binary classification the second order bound is tighter when $\E_{\rho^2}[\D(h,h')] > \E_\rho[L(h)]$. Below we provide a more detailed comparison of Theorems~\ref{thm:first-order} and \ref{thm:MV-bound} in the worst, the best, and the independent cases. The comparison only concerns the oracle bounds, whereas estimation of the oracle quantities, $\E_\rho[L(h)]$ and $\E_{\rho^2}[L(h,h')]$, is discussed in Section~\ref{sec:fast-vs-slow}.

\myparagraph{The worst case} Since $\E_{\rho^2}[L(h,h')] \leq \E_\rho[L(h)]$ the second order bound is at most twice worse than the first order bound. The worst case happens, for example, if all hypotheses in $\cal{H}$ give identical predictions. Then $\E_{\rho^2}[L(h,h')] = \E_\rho[L(h)] = L(\MV_\rho)$ for all $\rho$.

\myparagraph{The best case} Imagine that $\cal{H}$ consists of $M\geq 3$ hypotheses, such that each hypothesis errs on $1/M$ of the sample space (according to the distribution $D$) and that the error regions are disjoint. Then $L(h) = 1/M$ for all $h$ and $L(h,h') = 0$ for all $h\neq h'$ and $L(h,h)=1/M$. For a uniform distribution $\rho$ on $\cal{H}$ the first order bound is $2\E_\rho[L(h)] = 2/M$ and the second order bound is $4\E_{\rho^2}[L(h,h')] = 4/M^2$ and $L(\MV_\rho)=0$. In this case the second order bound is an order of magnitude tighter than the first order.

\myparagraph{The independent case} Assume that all hypotheses in $\cal{H}$ make independent errors and have the same error rate, $L(h) = L(h')$ for all $h$ and $h'$. Then for $h\neq h'$ we have $L(h,h') = \E_D[\1[h(X)\neq Y \wedge h'(X)\neq Y]] = \E_D[\1[h(X)\neq Y]\1[h'(X)\neq Y]] = \E_D[\1[h(X)\neq Y]]\E_D[\1[h'(X)\neq Y]] = L(h)^2$ and $L(h,h)=L(h)$. For a uniform distribution $\rho$ the second order bound is $4\E_{\rho^2}[L(h,h')] = 4(L(h)^2 + \frac{1}{M}L(h)(1-L(h)))$ and the first order bound is $2\E_{\rho}[L(h)] = 2L(h)$. Assuming that $M$ is large, so that we can ignore the second term in the second order bound, we obtain that it is tighter for $L(h) < 1/2$ and looser otherwise. The former is the interesting regime, especially in binary classification.  

In Appendix~\ref{app:alternative} we give additional intuition about Theorems~\ref{thm:first-order} and \ref{thm:MV-bound} by providing an alternative derivation.

\subsection{Comparison with the oracle C-bound}

The oracle C-bound is an alternative second order bound based on Chebyshev-Cantelli inequality (Theorem~\ref{thm:Chebyshev-Cantelli} in the appendix). It was first derived for binary classification by \citet[Theorem 2]{LLM+07} and several alternative forms were proposed by \citet[Theorem 11]{GLL+15}. \citet[Corollary 1]{LMRR17} extended the result to multiclass classification. 
To facilitate the comparison with our results we write the bound in terms of the tandem loss. 
In Appendix~\ref{app:C-bound-proof} we provide a direct derivation of Theorem~\ref{thm:C-bound} from Chebyshev-Cantelli inequality and in Appendix~\ref{app:equivalence} we show that it is equivalent to prior forms of the oracle C-bound.
\begin{theorem}[C-tandem Oracle Bound] 
If $\E_\rho[L(h)] < 1/2$, then
\[
L(\MV_\rho) \leq \frac{\E_{\rho^2}[L(h,h')] - \E_\rho[L(h)]^2}{\E_{\rho^2}[L(h,h')] - \E_\rho[L(h)] + \frac{1}{4}}.
\]
\label{thm:C-bound}
\end{theorem}
The theorem is essentially identical to the first form of oracle C-bound by \citet[Theorem 2]{LLM+07} and, as we show, it holds for multiclass classification.
In Appendix~\ref{app:Chebyshev-Cantelli} we show that the second order Markov's inequality behind Theorem~\ref{thm:MV-bound} is a relaxation of Chebyshev-Cantelli inequality. Therefore, the oracle C-bound is always at least as tight as the second order oracle bound in Theorem~\ref{thm:MV-bound}. In particular, \citeauthor{GLL+15} show that if the classifiers make independent errors and their error rates are identical and below 1/2, the oracle C-bound converges to zero with the growth of the number of classifiers, whereas, as we have shown above, the bound in Theorem~\ref{thm:MV-bound} only converges to $4L(h)^2$. However, the oracle C-bound has $\E_{\rho^2}[L(h,h')]$ and $\E_\rho[L(h)]$ in the denominator, which comes as a significant disadvantage in its estimation from data and minimization \citep{LIS19}, as we also show in our empirical evaluation.

\section{Second order PAC-Bayesian bounds for the weighted majority vote}\label{sec:pacbayes}

We apply PAC-Bayesian analysis to transform oracle bounds from the previous section into empirical bounds. The results are based on the following two theorems, where we use $\KL(\rho\|\pi)$ to denote the Kullback-Leibler divergence between distributions $\rho$ and $\pi$ and $\kl(p\|q)$ to denote the Kullback-Leibler divergence between two Bernoulli distributions with biases $p$ and $q$.

\begin{theorem}[PAC-Bayes-kl Inequality, \citealp{See02}] 
For any probability distribution $\pi$ on ${\cal H}$ that is independent of $S$ and any $\delta \in (0,1)$, with probability at least $1-\delta$ over a random draw of a sample $S$, for all distributions $\rho$ on ${\cal H}$ simultaneously:
\begin{equation}
\label{eq:PBkl}
\kl\lr{\E_\rho[\hat L(h,S)]\middle\|\E_\rho\lrs{L(h)}} \leq \frac{\KL(\rho\|\pi) + \ln(2 \sqrt n/\delta)}{n}.
\end{equation}
\label{thm:PBkl}
\end{theorem}
The next theorem provides a relaxation of the PAC-Bayes-kl inequality, which is more convenient for optimization. The upper bound is due to \citet{TIWS17} and the lower bound follows by an almost identical derivation, see Appendix~\ref{app:PBlambdaLower}. Both results are based on the refined Pinsker's lower bound for the kl-divergence. Since both the upper and the lower bound are deterministic relaxations of PAC-Bayes-kl, they hold simultaneously with no need to take a union bound over the two statements.

\begin{theorem}[PAC-Bayes-$\lambda$ Inequality, \citealp{TIWS17}]\label{thm:lambdabound} For any probability distribution $\pi$ on ${\cal H}$ that is independent of $S$ and any $\delta \in (0,1)$, with probability at least $1-\delta$ over a random draw of a sample $S$, for all distributions $\rho$ on ${\cal H}$ and all $\lambda \in (0,2)$ and $\gamma > 0$ simultaneously:
\begin{align}
\E_\rho\lrs{L(h)} &\leq \frac{\E_\rho[\hat L(h,S)]}{1 - \frac{\lambda}{2}} + \frac{\KL(\rho\|\pi) + \ln(2 \sqrt n/\delta)}{\lambda\lr{1-\frac{\lambda}{2}}n},\label{eq:PBlambda}\\
\E_\rho\lrs{L(h)} &\geq \lr{1 - \frac{\gamma}{2}}\E_\rho[\hat L(h,S)] - \frac{\KL(\rho\|\pi) + \ln(2 \sqrt n/\delta)}{\gamma n}.\label{eq:PBlambda-lower}
\end{align}
\label{thm:PBlambda}
\end{theorem}

\subsection{A general bound for multiclass classification}

We define the \emph{empirical tandem loss}
\[
\hat \L(h,h',S) = \frac{1}{n}\sum_{i=1}^n \1[h(X_i)\neq Y_i \wedge h'(X_i) \neq Y_i]
\]
and provide a bound on the expected loss of $\rho$-weighted majority vote in terms of the empirical tandem losses.
\begin{theorem}
\label{thm:tandem-lambda}
For any probability distribution $\pi$ on $\cal{H}$ that is independent of $S$ and any $\delta\in(0,1)$, with probability at least $1-\delta$ over a random draw of $S$, for all distributions $\rho$ on $\cal{H}$ and all $\lambda\in(0,2)$ simultaneously:
\[
L(\MV_\rho) \leq 4\lr{\frac{\E_{\rho^2}[\hat \L(h,h',S)]}{1-\lambda/2} + \frac{2\KL(\rho\|\pi) + \ln(2\sqrt n/\delta)}{\lambda(1-\lambda/2)n}}.
\]
\end{theorem}
\begin{proof}
The theorem follows by using the bound in equation~\eqref{eq:PBlambda} to bound $\E_{\rho^2}[L(h,h')]$ in Theorem~\ref{thm:MV-bound}. We note that $\KL(\rho^2\|\pi^2) = 2\KL(\rho\|\pi)$ \citep[Page 814]{GLL+15}.
\end{proof}
It is also possible to use PAC-Bayes-kl to bound $\E_{\rho^2}[L(h,h')]$ in Theorem~\ref{thm:MV-bound}, which actually gives a tighter bound, but the bound in  Theorem~\ref{thm:tandem-lambda} is more convenient for minimization. \citet{TS13} have shown that for a fixed $\rho$ the expression in Theorem~\ref{thm:tandem-lambda} is convex in $\lambda$ and has a closed-form minimizer. In Appendix~\ref{app:psd} we show that for fixed $\lambda$ and $S$ the bound is convex in $\rho$. Although in our applications $S$ is not fixed and the bound is not necessarily convex in $\rho$, a local minimum can still be efficiently achieved by gradient descent. 
A bound minimization procedure is provided in Appendix~\ref{app:minimization}.

\subsection{A specialized bound for binary classification}

We define the \emph{empirical disagreement}
\[
\hat \D(h,h',S') = \frac{1}{m} \sum_{i=1}^m \1[h(X_i)\neq h'(X_i)],
\]
where $S' = \lrc{X_1,\dots,X_m}$. The set $S'$ may have an overlap with the inputs X of the labeled set $S$, however, $S'$ may include additional unlabeled data.
The following theorem bounds the loss of weighted majority vote in terms of empirical disagreements. Due to possibility of using unlabeled data for estimation of disagreements in the binary case, the theorem has the potential of yielding a tighter bound when a considerable amount of unlabeled data is available. 
\begin{theorem}
\label{thm:disagreement}
In binary classification, for any probability distribution $\pi$ on $\cal{H}$ that is independent of $S$ and $S'$ and any $\delta\in(0,1)$, with probability at least $1-\delta$ over a random draw of $S$ and $S'$, for all distributions $\rho$ on $\cal{H}$ and all $\lambda\in(0,2)$ and $\gamma > 0$ simultaneously:
\begin{align*}
L(\MV_\rho) &\leq 4\lr{\frac{\E_\rho[\hat L(h,S)]}{1-\lambda/2} + \frac{\KL(\rho\|\pi) + \ln(4\sqrt n/\delta)}{\lambda(1-\lambda/2)n}}\\
&\qquad - 2\lr{(1-\gamma/2) \E_{\rho^2}[\hat \D(h,h',S')] - \frac{2\KL(\rho\|\pi) + \ln(4\sqrt m/\delta)}{\gamma m}}.
\end{align*}
\end{theorem}
\begin{proof}
The theorem follows by using the upper bound in equation~\eqref{eq:PBlambda} to bound $\E_\rho[L(h)]$ and the lower bound in equation~\eqref{eq:PBlambda-lower} to bound $\E_{\rho^2}[\D(h,h')]$ in Theorem~\ref{thm:MV-bound-binary}. We replace $\delta$ by $\delta/2$ in the upper and lower bound and take a union bound over them.
\end{proof}
Using PAC-Bayes-kl to bound $\E_\rho[L(h)]$ and $\E_{\rho^2}[\D(h,h')]$ in Theorem~\ref{thm:MV-bound-binary} gives a tighter bound, but the bound in Theorem~\ref{thm:disagreement} is more convenient for minimisation. The minimization procedure is provided in Appendix~\ref{app:minimization}.

\subsection{Ensemble construction}
\label{sec:validation}

\citet{TIWS17} have proposed an elegant way of constructing finite data-dependent hypothesis spaces that work well with PAC-Bayesian bounds. The idea is to generate multiple splits of a data set $S$ into pairs of subsets $S = T_h \cup S_h$, such that $T_h \cap S_h = \varnothing$. A hypothesis $h$ is then trained on $T_h$ and $\hat L(h,S_h)$ provides an unbiased estimate of its loss. The splits 
cannot depend on the data. Two examples of such splits are splits generated by cross-validation \citep{TIWS17} and splits generated by bagging in random forests, where out-of-bag (OOB) samples provide unbiased estimates of expected losses of individual trees \citep{LIS19}. It is possible to train multiple hypotheses with different parameters on each split,  as it happens in cross-validation. The resulting set of hypotheses produces an ensemble, and PAC-Bayesian bounds provide generalization bounds for a weighted majority vote of the ensemble and allow optimization of the weighting. There are two minor modifications required: the weighted empirical losses $\E_\rho[\hat L(h,S)]$ in the bounds are replaced by weighted validation losses $\E_\rho[\hat L(h,S_h)]$, and the sample size $n$ is replaced by the minimal validation set size $\nmin = \min_h |S_h|$. It is possible to use any data-independent prior, with uniform prior $\pi(h) = 1/|\cal{H}|$ being a natural choice in many cases \citep{TIWS17}.

For pairs of hypotheses $(h,h')$ we use the overlaps of their validation sets $S_h\cap S_{h'}$ to calculate an unbiased estimate of their tandem loss, $\hat L(h,h',S_h\cap S_{h'})$, which replaces $\hat L(h,h',S)$ in the bounds. The sample size $n$ is then replaced by $\nmin = \min_{h,h'} (S_h\cap S_{h'})$. 

\subsection{Comparison of the empirical bounds}
\label{sec:fast-vs-slow}

We provide a high-level comparison of the empirical first order bound ($\FO$), the new empirical second order bound based on the tandem loss ($\TND$, Theorem~\ref{thm:tandem-lambda}), and the new empirical second order bound based on disagreements ($\DIS$, Theorem~\ref{thm:disagreement}). The two key quantities in the comparison are the sample size $n$ in the denominator of the bounds and fast and slow convergence rates for the standard (first order) loss, the tandem loss, and the disagreements. \citet{TS13} have shown that if we optimize $\lambda$ for a given $\rho$, the PAC-Bayes-$\lambda$ bound in equation~\eqref{eq:PBlambda} can be written as
\[
\E_\rho[L(h)] \leq \E_\rho[\hat L(h,S)] + \sqrt{\frac{2\E_\rho[\hat L(h,S)]\lr{\KL(\rho\|\pi) + \ln(2\sqrt{n}/\delta)}}{n}} + \frac{2\lr{\KL(\rho\|\pi) + \ln(2\sqrt{n}/\delta)}}{n}.
\]
This form of the bound, introduced by \citet{McA03}, is convenient for explanation of fast and slow rates. If $\E_\rho[\hat L(h,S)]$ is large, then the middle term on the right hand side dominates the complexity and the bound decreases at the rate of $1/\sqrt{n}$, which is known as a \emph{slow rate}. If $\E_\rho[\hat L(h,S)]$ is small, then the last term dominates and the bound decreases at the rate of $1/n$, which is known as a \emph{fast rate}. 

\myparagraph{$\FO$ vs.\ $\TND$} The advantage of the $\FO$ bound is that the validation sets $S_h$ available for estimation of the first order losses $\hat L(h,S_h)$ are larger than the validation sets $S_h\cap S_{h'}$ available for estimation of the tandem losses. Therefore, the denominator $\nmin = \min_h |S_h|$ in the $\FO$ bound is larger than the denominator $\nmin = \min_{h,h'}|S_h\cap S_{h'}|$ in the $\TND$ bound. The $\TND$ disadvantage can be reduced by using data splits with large validation sets $S_h$ and small training sets $T_h$, as long as small training sets do not overly impact the quality of base classifiers $h$. Another advantage of the $\FO$ bound is that its complexity term has $\KL(\rho\|\pi)$, whereas the $\TND$ bound has $2\KL(\rho\|\pi)$. The advantage of the $\TND$ bound is that $\E_{\rho^2}[L(h,h')] \leq E_\rho[L(h)]$ and, therefore, the convergence rate of the tandem loss is typically faster than the convergence rate of the first order loss. The 
interplay of the estimation advantages and disadvantages, combined with the advantages and disadvantages of the underlying oracle bounds discussed in Section~\ref{sec:FOvsSO}, 
depends on the data and the hypothesis space.

\myparagraph{$\TND$ vs.\ $\DIS$} The advantage of the $\DIS$ bound relative to the $\TND$ bound is that in presence of a large amount of unlabeled data the disagreements $\D(h,h')$ can be tightly estimated (the denominator $m$ is large) and the estimation complexity is governed by the first order term, $\E_\rho[L(h)]$, which is "easy" to estimate, as discussed above. However, the $\DIS$ bound has two disadvantages. A minor one is its reliance on estimation of two quantities, $\E_\rho[L(h)]$ and $\E_{\rho^2}[\D(h,h')]$, which requires a union bound, e.g., replacement of $\delta$ by $\delta/2$. A more substantial one is that the disagreement term is desired to be large, and thus has a slow convergence rate. Since slow convergence rate relates to fast convergence rate as $1/\sqrt{n}$ to $1/n$, as a rule of thumb the $\DIS$ bound is expected to outperform $\TND$ only when the amount of unlabeled data is at least quadratic in the amount of labeled data, $m > n^2$.

\section{Empirical evaluation}
We studied the empirical performance of the bounds using standard random forests \citep{Bre01}
on a subset of data sets from the UCI and LibSVM repositories \citep{UCI,libsvm}. An overview of the data sets is given in Table~\ref{tab:data_sets} in the appendix. The number of points varied from 3000 to 70000 with dimensions $d<1000$. For each data set we set aside 20\% of the data for a test set $\testset$ and used the remaining data, which we call $S$, for ensemble construction and computation of the bounds. Forests with 100 trees were trained until leaves were pure, using the Gini criterion for splitting and considering $\sqrt{d}$ features in each split. We made 50 repetitions of each experiment and report the mean and standard deviation. In all our experiments $\pi$ was uniform and $\delta = 0.05$. We present two experiments: (1) a comparison of tightness of the bounds applied to uniform weighting, and (2) a comparison of weighting optimization the bounds. 
Additional experiments, where we explored the effect of using splits with increased validation and decreased training subsets, as suggested in Section~\ref{sec:fast-vs-slow}, and where we compared the $\TND$ and $\DIS$ bounds in presence of unlabeled data, are described in Appendix~\ref{app:experiments}.

The python source code for replicating the experiments is available at Github\footnote{\url{https://github.com/StephanLorenzen/MajorityVoteBounds}}.

\begin{figure}[t]
    \centering
    \includegraphics{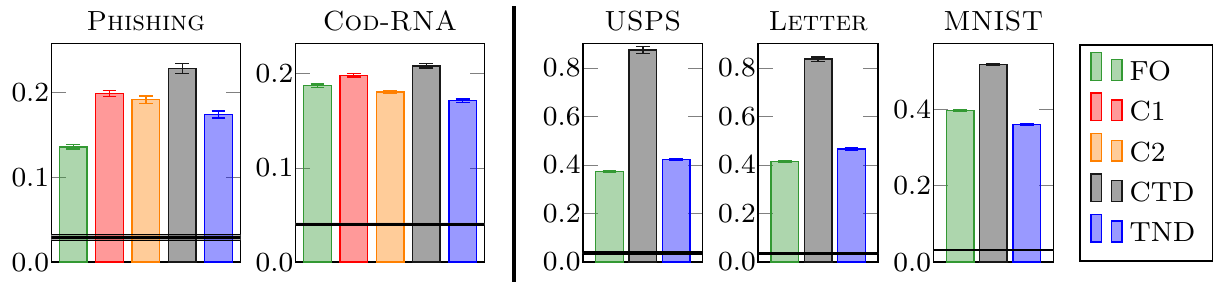}
    \caption{Test risk (black) and the bounds for a uniformly weighted random forest 
    on a subset of binary (left) and multiclass (right) datasets.
    Plots for the remaining datasets are provided in Figures~\ref{fig:uniform_all_bin} and \ref{fig:uniform_all_mul} in the appendix.
    }
    \label{fig:rf_example_bounds}
\end{figure}
\myparagraph{Uniform weighting} 
In Figure~\ref{fig:rf_example_bounds} we compare tightness of $\FO$, $\Cone$ and $\Ctwo$ (the two forms of C-bound by \citealp{GLL+15}, see Appendix~\ref{app:equivalence} for the oracle forms), the C-tandem bound ($\CTD$, Theorem~\ref{thm:C-bound}), and $\TND$ applied to uniformly weighted random forests 
on a subset of data sets. The right three plots are multiclass datasets, where $\Cone$ and $\Ctwo$ are inapplicable. 
The outcomes for the remaining datasets are reported in Figures~\ref{fig:uniform_all_bin} and \ref{fig:uniform_all_mul} in the appendix.
Since no optimization was involved, we used the PAC-Bayes-kl to bound $\E_\rho[L(h)]$, $\E_{\rho^2}[L(h,h')]$, and $\E_{\rho^2}[\D(h,h')]$ in the first and second order bounds, which is tighter than using PAC-Bayes-$\lambda$.
The $\TND$ bound was the tightest for 5 out of 16 data sets, and provided better guarantees than the C-bounds for 4 out of 7 binary data sets. In most cases, the $\FO$-bound was the tightest.


\myparagraph{Optimization of the weighting} We compared the loss on the test set $\testset$ and tightness after using the bounds for optimizing the weighting $\rho$. As already discussed, the C-bounds are not suitable for optimization (see also \citealp{LIS19}) and, therefore, excluded from the comparison. We used the PAC-Bayes-$\lambda$ form of the bounds for $\E_\rho[L(h)]$, $\E_{\rho^2}[L(h,h')]$, and $\E_{\rho^2}[\D(h,h')]$ for optimization of $\rho$ and then used the PAC-Bayes-kl form of the bounds for computing the final bound with the optimized $\rho$. Optimization details are provided in Appendix~\ref{app:minimization}.

\begin{figure}[t]
    \centering
    \begin{subfigure}{.49\linewidth}
     \centering
    \includegraphics[width=\linewidth]{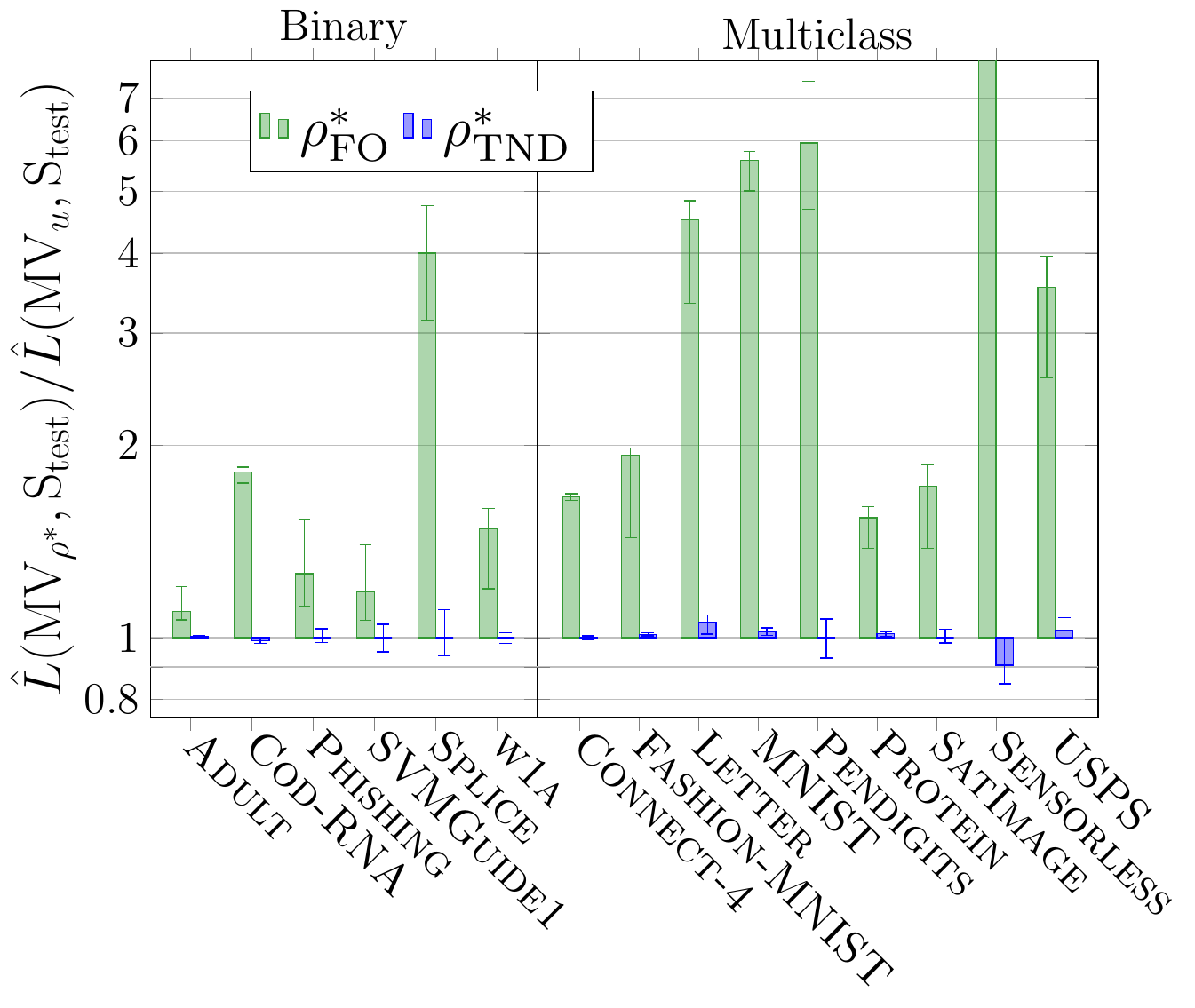}
    \caption{}
    \label{fig:opt_mvrisk}
    \end{subfigure}%
    \hfill
    \begin{subfigure}{.46\linewidth}
    \centering
    \includegraphics[width=\linewidth]{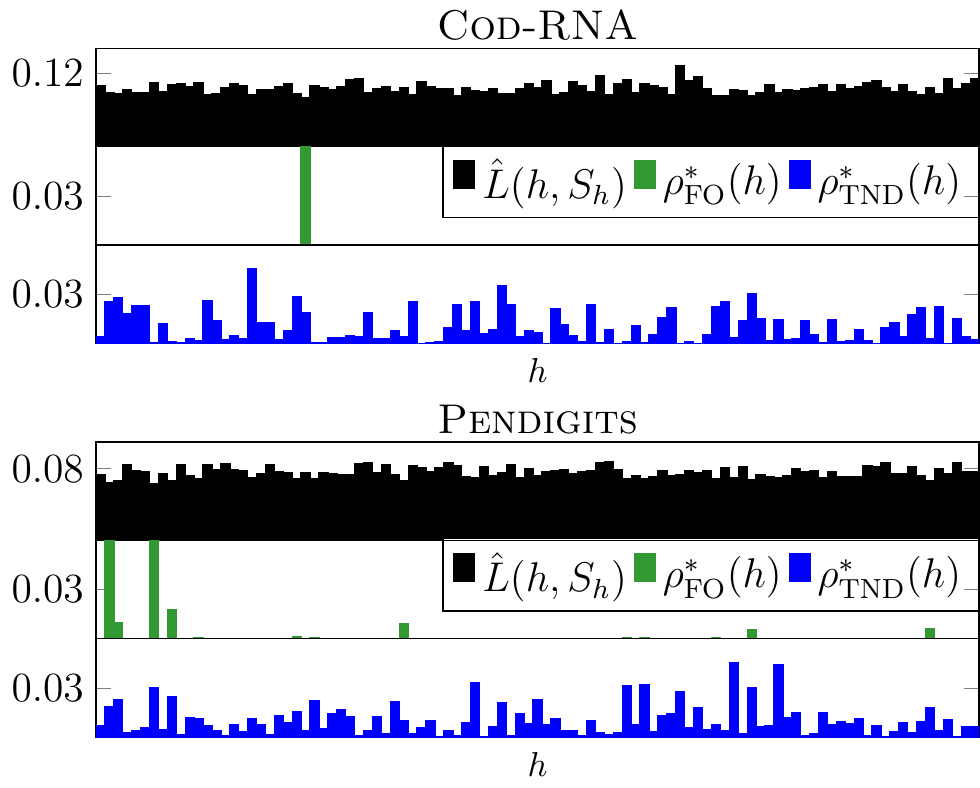}
    \caption{}
    \label{fig:example_rho}
    \end{subfigure}
    \caption{(a) The median, 25\%, and 75\% quantiles of the ratio $\hat L(\MV_{\rho^*}, \testset) / \hat L(\MV_u,\testset)$ of the test loss of majority vote with optimized 
    weighting $\rho^*$ generated by $\FO$ and $\TND$. The plot is on a logarithmic scale. Values above 1 represent degradation in performance on new data and values below 1 represent an improvement. (b) The optimized weights $\rho^*$ generated by $\FO$ and $\TND$.}
    \label{fig:opt_mvrisk_and_rho}
\end{figure}
Figure~\ref{fig:opt_mvrisk} compares the ratio of the loss of majority vote with optimized weighting to the loss of majority vote with uniform weighting on $\testset$ for $\rho^*$ found by minimization of $\FO$ and $\TND$. The numerical values are given in Table~\ref{tab:opt_mvrisk} in the appendix. 
While both bounds tighten with optimization, we observed that optimization of $\FO$ considerably weakens the performance on $\testset$ for all datasets, whereas optimization of $\TND$ did not have this effect and in some cases even improved the outcome. Figure~\ref{fig:example_rho} shows optimized distributions for two sample data sets. It is clearly seen that $\FO$ placed all the weight on a few top trees, while $\TND$ hedged the bets on multiple trees. The two figures demonstrate that the new bound correctly handled interactions between voters, as opposed to $\FO$.

\section{Discussion}

We have presented a new analysis of the weighted majority vote, which 
provides a reasonably tight generalization guarantee and can be used to guide optimization of the weights. The analysis has been applied 
to random forests, where the bound can be 
computed using out-of-bag samples 
with no need for a dedicated hold-out validation set, thus making highly efficient use of the data.
We have shown that in contrary to the commonly used first order bound, minimization of the new bound does not lead to deterioration of the test error, confirming that the analysis captures the cancellation of errors, which is the core of the majority vote.

\begin{ack}
We thank Omar Rivasplata and the anonymous reviewers for their suggestions for manuscript improvements. AM is funded by the Spanish Ministry of Science, Innovation and Universities under the projects TIN2016-77902-C3-3-P and PID2019-106758GB-C32, and by a Jose Castillejo scholarship CAS19/00279. SSL acknowledges funding by the Danish Ministry of Education and Science, Digital Pilot Hub and Skylab Digital. CI acknowledges support by the Villum Foundation through the project Deep Learning and Remote Sensing for Unlocking Global Ecosystem Resource Dynamics (DeReEco). YS acknowledges support by the Independent Research Fund Denmark, grant number 0135-00259B.
\end{ack}
\section*{Broader impact}
Ensemble classifiers,  in particular random forests, are among the most important tools in machine learning \citep{Delgado2014,Zhu15}, which are very frequently applied in practice \citep[e.g.,][]{chen2016xgboost,hoch2015ensemble,puurula2014kaggle,stallkamp:12}. Our study provides generalization guarantees for random forests and a method for tuning the weights of individual trees within a forest, which can lead to even higher accuracies. The result is of high practical relevance. 

Given that machine learning models are increasingly used to make decisions that have a strong impact on society, industry, and individuals, it is important that we have a good theoretical understanding of the employed methods and are able to provide rigorous guarantees for their performance. And here lies the strongest contribution of the line of research followed in our study, in which we derive rigorous bounds on the generalization error of random forests and other ensemble methods for multiclass classification.

\bibliography{bibliography,bibliography-Yevgeny}
\bibliographystyle{plainnat}

\newpage
\appendix

\renewcommand\thefigure{\thesection.\arabic{figure}} 
\renewcommand\thetable{\thesection.\arabic{table}} 
\renewcommand{\theequation}{\thesection.\arabic{equation}}
\renewcommand{\thetheorem}{\thesection.\arabic{theorem}}

\section{Proof of Lemmas~\ref{lem:second-monent} and \ref{lem:L2D}}
\label{app:L2D}

\begin{proof}[Proof of Lemma~\ref{lem:second-monent}]
\begin{align}
\E_D[\E_\rho[\1[h(X)\neq Y]]^2] &= \E_D[\E_\rho[\1[h(X)\neq Y]]\E_\rho[\1[h(X)\neq Y]]]\label{eq:MV-mid-bound}\\
&= \E_D[\E_{\rho^2}[\1[h(X)\neq Y]\1[h'(X)\neq Y]]]\notag\\
&= \E_D[\E_{\rho^2}[\1[h(X)\neq Y \wedge h'(X)\neq Y]]]\notag\\
&= \E_{\rho^2}[\E_D[\1[h(X)\neq Y \wedge h'(X)\neq Y]]]\notag\\
&= \E_{\rho^2}[\L(h,h')].\notag   
\end{align}
\end{proof}

\begin{proof}[Proof of Lemma~\ref{lem:L2D}]
Picking from \eqref{eq:MV-mid-bound}, we have
\begin{align*}
\E_\rho[\1[h(X)\neq Y]]\E_\rho[\1[h(X)\neq Y]] &= \E_\rho[\1[h(X)\neq Y]](1 - \E_\rho[(1 - \1[h(X)\neq Y])]\\
&= \E_\rho[\1[h(X)\neq Y]] - \E_\rho[\1[h(X)\neq Y]]\E_\rho[\1[h(X)= Y]]\\
& = \E_{\rho}[\1[h(X)\neq Y]] - \E_{\rho^2}[\1[h(X)\neq Y\wedge h'(X)= Y]]\\
&= \E_\rho[\1[h(X)\neq Y]] - \frac{1}{2}\E_{\rho^2}[\1[h(X)\neq h'(X)]].
\end{align*}
By taking expectation with respect to $D$ on both sides and applying Lemma~\ref{lem:second-monent} to the left hand side, we obtain:
\[
\E_{\rho^2}[L(h,h')] = \E_D[\E_\rho[\1[h(X)\neq Y]] - \frac{1}{2}\E_{\rho^2}[\1[h(X)\neq h'(X)]]] = \E_\rho[L(h)] - \frac{1}{2}\E_{\rho^2}[\D(h,h')].
\]
\end{proof}

\section{An alternative derivation of Theorems~\ref{thm:first-order} and \ref{thm:MV-bound} using relaxations of the indicator function}
\label{app:alternative}

\begin{figure}[ht]
\begin{center}
\includegraphics[scale=0.5]{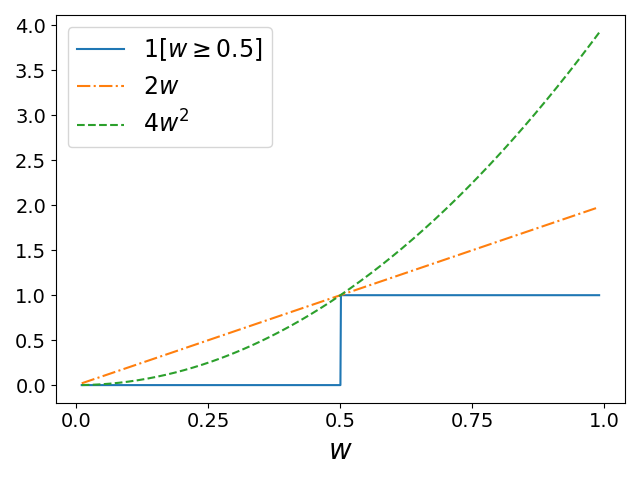}
\end{center}
\caption{\label{fig:indicator-bound} Relaxations of the indicator function.}
\end{figure}

In this section, we provide an alternative derivation of Theorems~\ref{thm:first-order} and \ref{thm:MV-bound} using relaxations of the indicator function. The alternative derivation may provide additional intuition about the method and this is how we initially have arrived to the results.

As explained in Section \ref{sec:generalsetup}, if majority vote makes an error, then at least a $\rho$-weighted half of the classifiers have made an error. Therefore, we have
$\ell(\MV_\rho(X),Y) \leq \1[\E_\rho[\1[h(X) \neq Y]] \geq 0.5]$. 

The first order bound can be derived from a first order relaxation of the indicator function. For any $w\in[0,1]$ we have $\1[w\geq0.5] \leq 2w$, see Figure~\ref{fig:indicator-bound}. Taking $w = \E_\rho[\1[h(X) \neq Y]]$ we have
\begin{align*}
L(\MV_\rho) &\leq \E_D[\1[\E_\rho[\1[h(X) \neq Y]] \geq 0.5]]\\
&\leq 2 \E_D[\E_\rho[\1[h(X)\neq Y]]] = 2\E_\rho[\E_D[\1[h(X)\neq Y]]] = 2\E_\rho[L(h)],
\end{align*}
which gives the result in Theorem~\ref{thm:first-order}.

The second order bound can be derived from a second order relaxation of the indicator function. We use the inequality $\1[w\geq0.5] \leq 4w^2$, which holds for all $w\in[0,1]$, see Figure~\ref{fig:indicator-bound}. As before, we take $w = \E_\rho[\1[h(X) \neq Y]]$. Then, we have
\[
L(\MV_\rho) \leq \E_D[\1[\E_\rho[\1[h(X) \neq Y]] \geq 0.5]] \leq 4\E_D[\E_\rho[\1[h(X)\neq Y]]^2] = 4\E_{\rho^2}[\L(h,h')],
\]
where the last equality is by Lemma~\ref{lem:second-monent}.

\section{Relation between second order Markov's and Chebyshev-Cantelli inequalities}
\label{app:Chebyshev-Cantelli}

In this section we show that second order Markov's inequality is a relaxation of Chebyshev-Cantelli inequality. In order to emphasize the relation between the proofs of Theorems~\ref{thm:first-order} and \ref{thm:MV-bound} in the body and in the previous section, we provide a direct derivation of Markov's and second order Markov's inequalities using relaxations of the indicator function. For any non-negative random variable $X$ and $\varepsilon > 0$ we have:
\begin{align*}
    \1[X\geq \varepsilon] &\leq \frac{1}{\varepsilon}X,\\
    \1[X \geq \varepsilon] &\leq \frac{1}{\varepsilon^2} X^2.
\end{align*}
We use these inequalities to recover the well-known Markov's inequality and prove the second order Markov's inequality.

\begin{theorem}[Markov's Inequality]
For a non-negative random variable $X$ and $\varepsilon > 0$
\[
\P[X \geq \varepsilon] \leq \frac{\E[X]}{\varepsilon}.
\]
\end{theorem}

\begin{proof}
\[
\P[X \geq \varepsilon] = \E[\1[X \geq \varepsilon]] \leq \frac{\E[X]}{\varepsilon}.
\]
\end{proof}

\begin{theorem}[Second order Markov's inequality]
For a non-negative random variable $X$ and $\varepsilon > 0$
\begin{equation}
\label{eq:RCC}
\P[X \geq \varepsilon] \leq \frac{\E[X^2]}{\varepsilon^2}.
\end{equation}
\end{theorem}

\begin{proof}
\[
\P[X \geq \varepsilon] = \E[\1[X \geq \varepsilon]] \leq \frac{\E[X^2]}{\varepsilon^2}.
\]
\end{proof}
We also cite Chebyshev-Cantelli inequality without a proof. For a proof see, for example, \citet{DGL96}.
\begin{theorem}[Chebyshev-Cantelli inequality]\label{thm:Chebyshev-Cantelli}
For a real-valued random variable $X$ and $\varepsilon > 0$
\begin{equation}
\label{eq:CC}
\P[X - \E[X] \geq \varepsilon] \leq \frac{\V[X]}{\varepsilon^2 + \V[X]},
\end{equation}
where $\V[X] = \E[X^2] - \E[X]^2$ is the variance of $X$.
\end{theorem}
Finally, we show that second order Markov's inequality is a relaxation of Chebyshev-Cantelli inequality.
\begin{lemma}
The second-order Markov's inequality \eqref{eq:RCC} is a relaxation of Chebyshev-Cantelli inequality \eqref{eq:CC}.
\end{lemma}
\begin{proof}
We show that inequality \eqref{eq:CC} is always at least as tight as inequality \eqref{eq:RCC}. The inequality \eqref{eq:RCC} is only non-trivial when $\E[X] < \varepsilon$, so for the comparison we can assume that $\E[X] < \varepsilon$. By \eqref{eq:CC} we then have:
\[
\P[X\geq\varepsilon] = \P[X - \E[X] \geq \varepsilon - \E[X]] \leq \frac{\V[X]}{(\varepsilon - \E[X])^2+\V[X]} = \frac{\E[X^2] - \E[X]^2}{\varepsilon^2 - 2\varepsilon\E[X]+\E[X^2]}
\]
Thus, we need to compare
\[
\frac{\E[X^2] - \E[X]^2}{\varepsilon^2 - 2 \varepsilon\E[X] + \E[X^2]} \text{~~~~~vs.~~~~~} \frac{\E[X^2]}{\varepsilon^2}.
\]
This is equivalent to the following row of comparisons:
\begin{align*}
(\E[X^2] - \E[X]^2)\varepsilon^2 &\text{~~~~~vs.~~~~~} \E[X^2] (\varepsilon^2 - 2 \varepsilon \E[X] + \E[X^2])\\
- \E[X]^2\varepsilon^2 &\text{~~~~~vs.~~~~~} \E[X^2] (- 2 \varepsilon \E[X] + \E[X^2])\\
0 &\text{~~~~~vs.~~~~~} \E[X]^2\varepsilon^2 - 2 \varepsilon \E[X] \E[X^2] + \E[X^2]^2\\
0&\text{~~~~~$\leq$~~~~~} (\E[X] \varepsilon - \E[X^2])^2,
\end{align*}
which completes the proof.
\end{proof}

\section{A proof of Theorem~\ref{thm:C-bound}}
\label{app:C-bound-proof}

We provide a direct proof of Theorem~\ref{thm:C-bound} using Chebyshev-Cantelli inequality.

\begin{proof}
We apply Chebyshev-Cantelli inequality to $\E_\rho[\1(h(X)\neq Y)]$:
\begin{align*}
L(\MV_\rho) &\leq \P[\E_\rho[\1(h(X)\neq Y)] \geq \frac{1}{2}]\\
&= \P[\E_\rho[\1(h(X)\neq Y)] - \E_\rho[L(h)] \geq \frac{1}{2} - \E_\rho[L(h)]]\\
&\leq \frac{\E_{\rho^2}[L(h,h')] - \E_\rho[L(h)]^2}{\lr{\frac{1}{2} - \E_\rho[L(h)]}^2 + \E_{\rho^2}[L(h,h')] - \E_\rho[L(h)]^2}\\
&= \frac{\E_{\rho^2}[L(h,h')] - \E_\rho[L(h)]^2}{\E_{\rho^2}[L(h,h')] - \E_\rho[L(h)] + \frac{1}{4}}.
\end{align*}
\end{proof}

\section{Equivalence of Theorem~\ref{thm:C-bound} to prior forms of the oracle C-bound}
\label{app:equivalence}

In this section we show that the C-tandem oracle bound in Theorem~\ref{thm:C-bound} is equivalent to prior forms of the oracle C-bound.

\subsection{Equivalence to Corollary 1 of \citet{LMRR17}}

\citeauthor{LMRR17}\ write their oracle C-bound in terms of an $\omega$-margin, denoted  by $M_{\rho,\omega}$(X,Y), which is defined as $M_{\rho,\omega} (X,Y)= \E_{\rho}[\1(h(X)=Y)] - \omega$, or, equivalently, $M_{\rho,\omega} (X,Y)
= (1-\omega) - \E_{\rho}[\1(h(X)\neq Y)]$. By simple algebraic manipulations we have the following identities, which show the equivalence to Theorem~\ref{thm:C-bound}:

\begin{align*}
\underbrace{\frac{\E_{\rho^2}[L(h,h')] - \E_\rho[L(h)]^2}{\E_{\rho^2}[L(h,h')] - \E_\rho[L(h)] + \frac{1}{4}}}_{\text{C-tandem oracle}}
=1-\frac{\frac{1}{4}-\E_\rho[L(h)] + \E_\rho[L(h)]^2}{\frac{1}{4} - \E_\rho[L(h)] + \E_{\rho^2}[L(h,h')]}
= \underbrace{1-\frac{\big(\E_D[M_{\rho,\frac{1}{2}}(X,Y)]\big)^2}{\E_D[(M_{\rho,\frac{1}{2}}(X,Y))^2]}}_{\text{Oracle C-bound of \citeauthor{LMRR17}}}.
\end{align*}

\subsection{Equivalence to Theorem 11 of \citet{GLL+15} in binary classification}

In binary classification we can apply Lemma~\ref{lem:L2D} and simple algebraic manipulations to obtain the following identities, which demonstrate equivalence of Theorem~\ref{thm:C-bound} and Theorem 11 of \citeauthor{GLL+15}:
\begin{align*}
\underbrace{\frac{\E_{\rho^2}[L(h,h')] - \E_\rho[L(h)]^2}{\E_{\rho^2}[L(h,h')] - \E_\rho[L(h)] + \frac{1}{4}}}_{\text{C-tandem oracle}} 
&= \frac{4\E_\rho[L(h)] - 4 (\E_\rho[L(h)])^2 - 2 \E_{\rho^2}[\D(h,h')]}{1 - 2\E_{\rho^2}[\D(h,h')]}\\
&= \underbrace{1 - \frac{\lr{1 - 2 \E_\rho[L(h)]}^2}{1 - 2 \E_{\rho^2}[\D(h,h')]}}_{\text{C1 oracle}}\\
&= \underbrace{1 - \frac{\lr{1 - (2 \E_{\rho^2}[L(h,h')] + \E_{\rho^2}[\D(h,h')])}^2}{1 - 2 \E_{\rho^2}[\D(h,h')]}}_{\text{C2 oracle}}.
\end{align*}
The second line is the oracle form of C1 bound of \citeauthor{GLL+15} and the last line is the oracle form of their C2 bound. 

We note that while all forms of the oracle C-bound are equivalent, their translation into empirical bounds might have different tightness due to varying difficulty of estimation of the oracle quantities $\E_\rho[L(h)]$, $\E_{\rho^2}[L(h,h')]$, and $\E_{\rho^2}[\D(h,h')]$, as discussed in Section~\ref{sec:fast-vs-slow}.

\section{A proof of Theorem~\ref{thm:PBlambda}}
\label{app:PBlambdaLower}

We provide a proof of the lower bound \eqref{eq:PBlambda-lower} in Theorem~\ref{thm:PBlambda}. The upper bound \eqref{eq:PBlambda} has been shown by \citet{TIWS17}. The proof of the lower bound follows the same steps as the proof of the upper bound.
\begin{proof}
We use the following version of refined Pinsker's inequality \citep[Lemma 8.4]{Mar96,Mar97,Sam00,BLM13}: for $p > q$
\begin{equation}
\label{eq:kl}
\kl(p\|q) \geq (p-q)^2/(2p).
\end{equation}
By application of inequality \eqref{eq:kl}, inequality \eqref{eq:PBkl} can be relaxed to
\begin{equation}
\E_\rho\lrs{\hat L(h,S)} - \E_\rho\lrs{L(h)} \leq \sqrt{2 \E_\rho\lrs{\hat L(h,S)} \frac{\KL(\rho\|\pi) + \ln \frac{2 \sqrt n}{\delta}}{n}}.
\label{eq:PBsqrt}
\end{equation}
By using the inequality $\sqrt{xy} \leq \frac{1}{2}\lr{\gamma x + \frac{y}{\gamma}}$ for all $\gamma > 0$, we have that with probability at least $1-\delta$ for all $\rho$ and $\gamma > 0$
\begin{equation}
\E_\rho\lrs{\hat L(h,S)} - \E_\rho\lrs{L(h)} \leq \frac{\gamma}{2}\E_\rho\lrs{\hat L(h,S)} + \frac{\KL(\rho\|\pi) + \ln \frac{2 \sqrt n}{\delta}}{\gamma n}.
\end{equation}
By changing sides
\[
\E_\rho\lrs{L(h)} \geq \lr{1 - \frac{\gamma}{2}}  \E_\rho\lrs{\hat L(h,S)} - \frac{\KL(\rho\|\pi) + \ln \frac{2 \sqrt n}{\delta}}{\gamma n}.
\]
\end{proof}

\section{Positive semi-definiteness of the matrix of empirical tandem losses}\label{app:psd}

In Lemma~\ref{lem:psd} below we show that if the empirical tandem losses are evaluated on the same set $S$, then the matrix of empirical tandem losses $\hatLtnd$ with entries $(\hatLtnd)_{h,h'} = \hat L(h,h',S)$ is positive semi-definite. This implies that for a fixed $\lambda$ the bound in Theorem~\ref{thm:tandem-lambda} is convex in $\rho$, because in this case $\E_\rho^2[\hat L(h,h',S)] = \rho^T \hatLtnd \rho$ is convex in $\rho$ and $\KL(\rho\|\pi)$ is always convex in $\rho$. (We note, however, that the bound is not necessarily jointly convex in $\rho$ and $\lambda$ and, therefore, alternating minimization of the bound may still converge to a local minimum. While \citet{TIWS17} derive conditions under which the PAC-Bayes-$\lambda$ bound for the first order loss is quasiconvex, such analysis of the bound for the second order loss would be more complicated.) In Section~\ref{sec:non-psd} we then provide an example showing that if the tandem losses are evaluated on different sets, as it happens in our case, where the entries are $(\hatLtnd)_{h,h'} = \hat L(h,h',S_h\cap S_{h'})$, then the matrix of tandem losses is not necessarily positive semi-definite. Therefore, in our case minimization of the bound is only expected to converge to a local minimum.

\begin{lemma}\label{lem:psd}
Given $\numhyp$ hypotheses and data $S=\{(X_1,Y_1),\dots,(X_n,Y_n)\}$,
the $\numhyp\times\numhyp$ matrix of empirical tandem losses $\hatLtnd$ with 
entries  $(\hatLtnd)_{h,h'}=\hat \L(h,h',S)$ is positive semi-definite.
\end{lemma}
\begin{proof}
Define a vector of empirical losses by hypotheses in $\cal{H}$ on a sample $(X_i,Y_i)$ by
\[
\hat \ell_i = 
\begin{pmatrix} \1[h_1(X_i)\neq Y_i]  \\ \vdots \\ \1[h_\numhyp(X_i)\neq Y_i]  \end{pmatrix}.
\]
Then the $(h,h')$ entry of the matrix $\hat \ell_i \hat \ell_i^T$ is $(\hat \ell_i \hat \ell_i^T)_{h,h'} = \1[h(X_i) \neq Y_i]\1[h'(X_i)\neq Y_i]$. Thus, the matrix of empirical tandem losses can 
be written as a mean of outer products
\begin{equation*}
 \hatLtnd= \frac{1}{n}\sum_{i=1}^n \hat \ell_i \hat \ell_i^T
\end{equation*}
and is, therefore, positive semi-definite.
\end{proof}

\subsection{Non positive semi-definite example}
\label{sec:non-psd}

If the empirical tandem losses are estimated on different subsets of the data rather than a common set $S$, as in the case of out-of-bag samples, where we take $\hat L(h, h', S_h \cap S_{h'})$, the resulting matrix of empirical tandem losses is not necessarily positive semi-definite.
Consider the following example with 2 points, 3 hypotheses, and the following losses:
\begin{center}
\begin{tabular}{c|cc}
     & $X_1$ & $X_2$ \\
     \hline
    $h_1$ & 1 & 0\\
    $h_2$ & 0 & 1\\
    $h_3$ & 0 & 1
\end{tabular}
\end{center}
If we compute the tandem loss for $h_1$ and $h_2$ on the first point and the tandem loss for $h_1$ and $h_3$ and for $h_2$ and $h_3$ on the second point, and the tandem losses of hypotheses with themselves on all the points, then we have
\[
\hat L(h,h') = \lr{\begin{array}{ccc}
    0.5 & 0 & 0 \\
    0 & 0.5 & 1 \\
    0 & 1 & 0.5
\end{array}}.
\]
This matrix is not positive semi-definite, it has eigenvalues $-0.5, 0.5$, and $1.5$.

\section{Gradient-based minimization of the bounds}
\label{app:minimization}

This section gives details on the optimization of the bounds in Theorems~\ref{thm:tandem-lambda} and \ref{thm:disagreement}.
First, we consider the bound in Theorem~\ref{thm:tandem-lambda} and provide a closed form solution for the parameter $\lambda$ given $\rho$ as well as the gradient of the bound w.r.t.{} $\rho$ for fixed $\lambda$.
Then we give the closed form solutions for the parameters $\lambda$ and $\gamma$ given $\rho$ and  the gradient  w.r.t.{} $\rho$ for fixed $\lambda$ and $\gamma$ for the bound in Theorem~\ref{thm:disagreement}.
After that, we describe the alternating minimization procedure we applied for optimization in our experiments. 
\subsection{Minimization of the bound in Theorem~\ref{thm:tandem-lambda}}

\paragraph{Optimal $\lambda$ given $\rho$}

Given $\rho$, the optimal $\lambda$ in Theorem~\ref{thm:tandem-lambda} can be computed  following \citet{TS13} and \citet{TIWS17}, because the optimization problem is the same:
\[
\lambda = \frac{2}{\sqrt{\frac{2n\E_{\rho^2}[\hat L(h,h',S)]}{2\KL(\rho\|\pi)+\ln\frac{2\sqrt n}{\delta}} + 1} + 1}.
\]

\paragraph{Gradient with respect to $\rho$ given $\lambda$}
Next we calculate the  gradient for minimizing the bound in Theorem~\ref{thm:tandem-lambda} with respect to $\rho$ under fixed $\lambda$. The minimization is equivalent to minimizing $f(\rho) = \E_{\rho^2}[\hat L(h,h',S)] + \frac{2}{\lambda n}\KL(\rho\|\pi)$ under the constraint that $\rho$ is a probability distribution. Let $(\nabla f)_h$ for $h\in\cal{H}$  denote the component of the gradient corresponding to hypothesis $h$. We also use $\hatLtnd$ to denote the matrix of empirical tandem losses and $\ln \frac{\rho}{\pi}$ to denote the vector with entry corresponding to hypothesis $h$ being $\ln \frac{\rho(h)}{\pi(h)}$. We have:
\begin{align*}
(\nabla f)_h &= 2\sum_{h'}\rho(h') \hat L(h,h',S) + \frac{2}{\lambda n}\lr{1 + \ln \frac{\rho(h)}{\pi(h)}},\\
\nabla f &= 2 \lr{\hatLtnd \rho + \frac{1}{\lambda n} \lr{1+\ln \frac{\rho}{\pi}}}.
\end{align*}

\subsection{Minimization of the bound in Theorem~\ref{thm:disagreement}}

\paragraph{Optimal $\lambda$ and $\gamma$ given $\rho$} The optimal $\lambda$ can be computed as above, because the optimization problem is the same. The only difference is that we have $\delta/2$ instead of $\delta$:
\[
\lambda = \frac{2}{\sqrt{\frac{2n\E_\rho[\hat L(h,S)]}{\KL(\rho\|\pi)+\ln\frac{4\sqrt n}{\delta}} + 1} + 1}.
\]
Minimization of the bound in Theorem~\ref{thm:disagreement} with fixed $\rho$ with respect to $\gamma$ is equivalent to minimizing $\frac{\gamma}{2}a + \frac{b}{\gamma}$ with $a = \E_{\rho^2}[\hat \D(h,h',S')]$ and $b = \frac{2\KL(\rho\|\pi) + \ln(4\sqrt m / \delta)}{m}$. The minimum is achieved by $\gamma=\sqrt{\frac{2b}{a}}$:
\[
\gamma = \sqrt{\frac{4\KL(\rho\|\pi) + \ln(16 m/\delta^2)}{m\E_{\rho^2}[\hat D(h,h',S')]}}.
\]

\paragraph{Gradient with respect to $\rho$}

Minimization of the bound with respect to $\rho$ for fixed $\lambda$ and $\gamma$ is equivalent to constrained minimization of $f(\rho) = 2a\E_\rho[\hat L(h,S)] - b \E_{\rho^2}[\hat \D(h,h',S')] + 2c\KL(\rho\|\pi)$, where $a = \frac{1}{1-\lambda/2}$, $b = 1-\gamma/2$, and $c = \frac{1}{\lambda(1-\lambda/2)n} + \frac{1}{\gamma m}$, and the constraint is that $\rho$ is a probability distribution. We use $\hat L$ to denote the vector of empirical losses of $h \in \cal{H}$ and $\hat \D$ to denote the matrix of empirical disagreements. We have:
\begin{align*}
    (\nabla f)_h &= 2a \hat L(h,S) - 2b \sum_{h'} \rho(h') \D(h,h',S) + 2c\lr{1+\ln \frac{\rho}{\pi}},\\
    \nabla f &= 2\lr{a \hat L - b \hat \D \rho + c\lr{1+\ln \frac{\rho}{\pi}}}.
\end{align*}

\subsection{Alternating optimization procedure}
In our experiments, we applied an alternating optimization procedure to improve the weighting $\rho$ of the ensemble members as well as the parameters $\lambda$ and, when considering the disagreement, $\gamma$.

Let $\numhyp=|\cal{H}|$ denote the number of ensemble members.
We parameterize $\rho$
by $\tilde{\rho}\in\mathbb{R}^{\numhyp}$ with 
$\rho=\operatorname{softmax}(\tilde{\rho})$, where $\rho_i = \frac{\exp{\tilde{\rho}_i}}{\sum_{j=1}^{\numhyp}  \exp{\tilde{\rho}_j} }$ for $i=1,\dots,\numhyp$. This ensures that $\rho$ is a proper probability distribution and allows us to apply unconstrained optimization in the adaption of $\rho$. Because  we are using uniform priors $\pi$ and due to the regularization in terms of the  Kullback-Leibler divergence between $\rho$ and $\pi$ in the bounds, excluding  $\rho_i\in\{0,1\}$ for each $i=1,\dots,B$ is not a limitation.  

Starting from uniform $\rho$ and the corresponding optimal $\lambda$ and, if applicable, $\gamma$, we looped through the following steps: We applied iterative gradient-based optimization of $\rho$ parameterized by $\tilde{\rho}$ until the bound did not improve for 10 iterations.
Then we computed  the optimal $\lambda$ and, in the case of the $\DIS$ bound, $\gamma$ for the optimized $\rho$. We stopped if the change in the bound  was smaller than $10^{-9}$.
We applied iRProp$^+$  for the gradient based optimization, a first order method with adaptive individual step sizes \citep{igel:01e,florescu:18}.

\section{Experiments}
\label{app:experiments}

This section provides details on the data sets used in the experiments and provides details, additional figures, and numerical values for the empirical evaluations: empirical evaluation of the bounds using a standard random forest with uniform weighting (Section~\ref{apx:rf-bagging}, expanding the first experiment and  Figure~\ref{fig:rf_example_bounds} in the body), and optimization of the weighting of the trees (Section~\ref{apx:bagging-optimized}, expanding the second experiment and  Figure~\ref{fig:opt_mvrisk_and_rho} in the body). We also include additional experiments with \emph{reduced bagging}, where we use less data for construction of each tree in order to leave larger out-of-bag sets for improved estimation of the second order quantities. The diagram below provides an overview of the experiments with references to the relevant subsections.
\begin{figure}[ht]
    \centering
    \begin{tikzpicture}
  \tikzstyle{heading} = [rectangle,draw, align=center, text width=3.6cm, minimum height=6mm, inner sep=0, outer sep=0,fill=black,text=white]
  \tikzstyle{subheading} = [rectangle,draw, align=center, text width=1.8cm, minimum height=6mm, inner sep=0, outer sep=0]
  \tikzstyle{cell} = [rectangle,draw, align=center, minimum height=1.8cm, text width=1.8cm, inner sep=0, outer sep=0]
  \tikzstyle{comp} = [{Triangle[length=3mm, width=7mm]}-{Triangle[length=3mm, width=7mm]}, line width=5mm, font=\footnotesize]
  \matrix
          {
    \node[draw=none,fill=none] {}; &
    \node[subheading] (topleft) {Full}; &
    \node[subheading] {Reduced};\\
    \node[subheading, rotate=90] (lefttop) {Uniform}; &
    \node[cell] (FulUniform) {}; &
    \node[cell] (RedUniform) {};\\
    \node[subheading, rotate=90] {Optimized}; &
    \node[cell] (FulOpt) {}; &
    \node[cell] (RedOpt) {};\\
          };
          \node[heading, anchor=south west] at (topleft.north west) {\textbf{Bagging}};
          \node[heading, anchor=south east, rotate=90] at (lefttop.north east) {\textbf{Weights}};
          \draw [comp] ([yshift=-3mm]FulUniform.center) -- ([yshift=3mm]FulOpt.center) node[midway,color=white] {\large{\textbf{A}}};
          \draw [comp] ([xshift=3mm]FulUniform.center) -- ([xshift=-3mm]RedUniform.center) node[midway,color=white] {\large{\textbf{B}}};
          \draw [comp] ([xshift=3mm]FulOpt.center) -- ([xshift=-3mm]RedOpt.center) node[midway,color=white] {\large{\textbf{C}}};
\end{tikzpicture}
\end{figure}
\begin{itemize}
    \item[\textbf{A}] Comparison of uniformly weighted random forests and random forests with optimized weighting in the full bagging setting: 
    Section~\ref{apx:bagging-optimized}, expanding on the experiments in the body of the paper.
    \item[\textbf{B}] Comparison of uniformly weighted random forests with standard (full) and reduced bagging: Section~\ref{apx:rf-reduced}.     \item[\textbf{C}] Comparison of random forests with optimized weighting in the full and reduced bagging settings:
    Section~\ref{apx:rf-reduced}
\end{itemize}
For each experiment, we report the mean and standard deviations of 50 runs. We used standard random forests trained on $S$ (80\% of the data) and evaluated on test set $\testset$ (20\%). 100 trees were used for each data set, and $\sqrt{d}$ features were considered in each split. The bounds were evaluated on the OOB data, with uniform $\pi$ and $\delta=0.05$. 

Furthermore, Section~\ref{apx:unlabeled} presents an empirical evaluation of the $\DIS$ bound in the setting with only a small amount of labeled data available and large amounts of unlabeled data. For this experiment, we reserved part of $S$ as unlabeled data and evaluated $\FO$, $\TND$ and $\DIS$. We varied the split between labeled training data and unlabeled data and report the means and standard deviations of 20 runs for each split.

\subsection{Data sets}
As mentioned, we considered data sets from the UCI and LibSVM repositories \citep{UCI,libsvm}, as well as \dataset{Fashion-MNIST} from Zalando Research\footnote{\url{https://research.zalando.com/welcome/mission/research-projects/fashion-mnist/}}. We used data sets with size $3000 \leq N \leq 70000$ and dimension $d \leq 1000$. These relatively large data sets were chosen in order to provide meaningful bounds in the standard bagging setting, where individual trees are trained on $n=0.8N$ randomly subsampled points with replacement and the size of the overlap of out-of-bag sets is roughly $n/9$.  
An overview of the data sets is given in Table~\ref{tab:data_sets}.
\begin{table}[t]
    \centering
    \caption{Data set overview. $c_{\min}$ and $c_{\max}$ denote the minimum and maximum class frequency.}
    \label{tab:data_sets}
    \begin{tabular}{@{}lcccccc@{}}\toprule
Dataset & $N$ & $d$ & $c$ & $c_{\min}$ & $c_{\max}$ & Source\\
\midrule
\dataset{Adult} & 32561 & 123 & 2 & 0.2408 & 0.7592 & LIBSVM (a1a) \\
\dataset{Cod-RNA} & 59535 & 8 & 2 & 0.3333 & 0.6667 & LIBSVM \\
\dataset{Connect-4} & 67557 & 126 & 3 & 0.0955 & 0.6583 & LIBSVM \\
\dataset{Fashion-MNIST} & 70000 & 784 & 10 & 0.1000 & 0.1000 & Zalando Research \\
\dataset{Letter} & 20000 & 16 & 26 & 0.0367 & 0.0406 & UCI \\
\dataset{MNIST} & 70000 & 780 & 10 & 0.0902 & 0.1125 & LIBSVM \\
\dataset{Mushroom} & 8124 & 22 & 2 & 0.4820 & 0.5180 & LIBSVM \\
\dataset{Pendigits} & 10992 & 16 & 10 & 0.0960 & 0.1041 & LIBSVM \\
\dataset{Phishing} & 11055 & 68 & 2 & 0.4431 & 0.5569 & LIBSVM \\
\dataset{Protein} & 24387 & 357 & 3 & 0.2153 & 0.4638 & LIBSVM \\
\dataset{SVMGuide1} & 3089 & 4 & 2 & 0.3525 & 0.6475 & LIBSVM \\
\dataset{SatImage} & 6435 & 36 & 6 & 0.0973 & 0.2382 & LIBSVM \\
\dataset{Sensorless} & 58509 & 48 & 11 & 0.0909 & 0.0909 & LIBSVM \\
\dataset{Shuttle} & 58000 & 9 & 7 & 0.0002 & 0.7860 & LIBSVM \\
\dataset{Splice} & 3175 & 60 & 2 & 0.4809 & 0.5191 & LIBSVM \\
\dataset{USPS} & 9298 & 256 & 10 & 0.0761 & 0.1670 & LIBSVM \\
\dataset{w1a} & 49749 & 300 & 2 & 0.0297 & 0.9703 & LIBSVM \\
\bottomrule
\end{tabular}

\end{table}

For all experiments, we removed patterns with missing entries and made a stratified split of the data set.
For data sets with a training and a test set (\dataset{SVMGuide1},
\dataset{Splice},
\dataset{Adult},
\dataset{w1a},
\dataset{MNIST},
\dataset{Shuttle},
\dataset{Pendigits},
\dataset{Protein},
\dataset{SatImage},
\dataset{USPS})
we combined the training and test sets and shuffled the entire set before splitting.
\subsection{Standard uniformly weighted random forests}\label{apx:rf-bagging}
This section provides additional figures and numerical values of the bounds computed for the standard uniformly weighted random forest using bagging (Figure~\ref{fig:rf_example_bounds} in the body), as well as additional statistics for the experiments. 

\begin{figure}[ht]
    \centering
    \includegraphics[width=\linewidth]{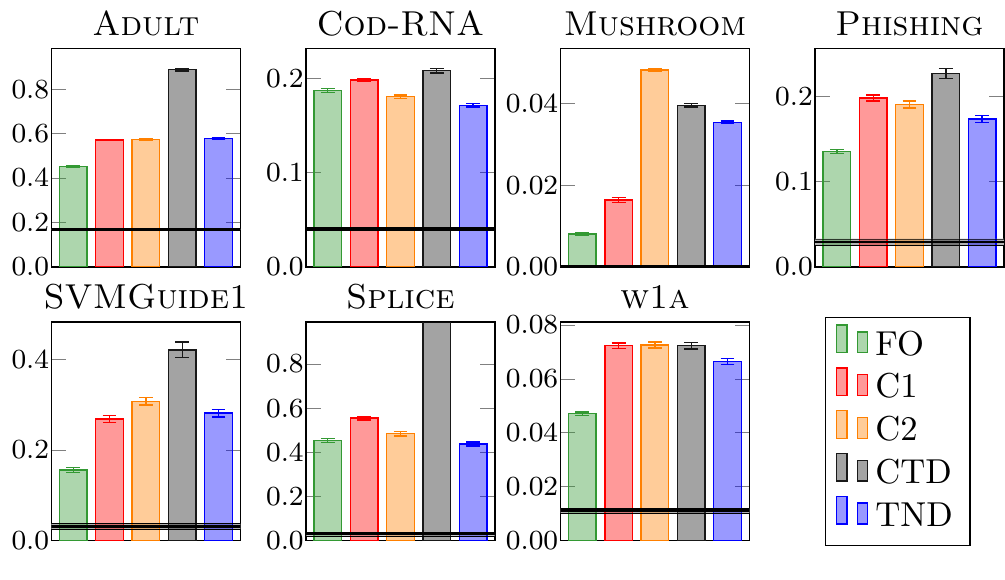}
    \caption{Plot of the bounds for binary data sets with the standard uniformly weighted random forests.
    The test losses are depicted by black lines.}
    \label{fig:uniform_all_bin}
\end{figure}
\begin{figure}[ht]
    \centering
    \includegraphics[width=\linewidth]{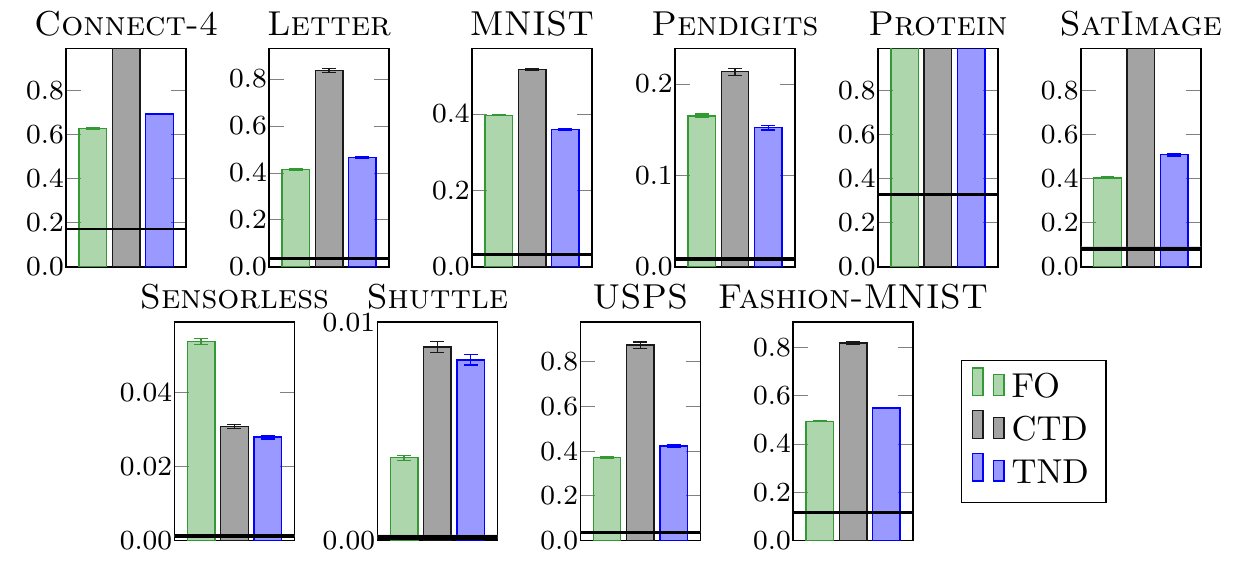}
    \caption{Plot of the bounds for multiclass data sets with standard uniformly weighted random forests. 
    The test losses are depicted by black lines.}
    \label{fig:uniform_all_mul}
\end{figure}
Figures~\ref{fig:uniform_all_bin} and \ref{fig:uniform_all_mul} plot the bounds obtained by the standard random forest for the binary and multiclass data sets respectively.
Table~\ref{tab:rf_bagging_bounds} reports the means and standard deviations for all data sets. Additional information (randomized loss, tandem loss, etc.) is reported in Table~\ref{tab:rf_bagging_info}.

$\TND$ is tightest for 2 out of 7 binary data sets and 3 out of 10 multiclass data sets, while $\FO$ is tightest for the rest. Figure \ref{fig:ratiofDisagreement} plots the ratio between the empirical disagreement $\E_{\rho^2}[\hat \D(h,h',S_h\cap S_{h'})]$ and the empirical randomized loss $\E_\rho[\hat L(h,S_h)]$ versus the ratio between the TND and FO bounds. This figure shows that $\TND$ bound tends to be tighter than $\FO$ when the disagreement is large in relation to the randomized loss. Since the amounts of data $|S_h \cap S_{h'}|$ available for estimation of the tandem losses are considerably smaller than the amounts of data $|S_h|$ available for estimation of the first order losses, the empirical disagreement has to be considerably larger than the empirical loss for $\TND$ to take the advantage over $\FO$. This is in agreement with the discussion provided in Sections \ref{sec:FOvsSO} and \ref{sec:fast-vs-slow}. 



Comparing $\TND$ to the other second order bounds, we see that $\TND$ is tighter (or almost as tight) in all cases, except for \dataset{Mushroom}, where $\Cone$ is tighter. This is due to $\Cone$ being given in terms of an upper bound on $\E_\rho[\L(h)]$ and a lower bound on $\E_{\rho^2}[\D(h,h')]$. With the lower bound being almost zero
, we have $\Cone \approx 2\FO$ and since the disagreement is very low, $\TND \approx 4\FO$. We note that even though $\Cone$ is tighter than $\TND$ in this case, it is still much weaker than $\FO$, because, as it has been discussed in Section~\ref{sec:FOvsSO}, problems with low disagreement are not well-suited for second order bounds. 

\begin{figure}[ht]
    \centering
    \includegraphics[width=0.75\linewidth]{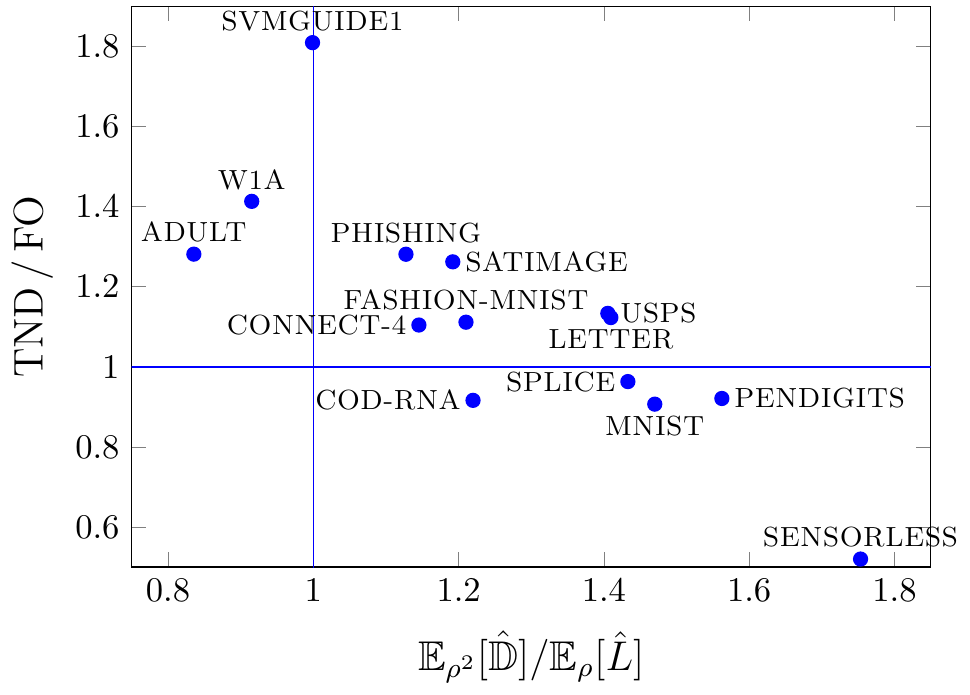}
    \caption{Ratio between the empirical disagreement $\E_{\rho^2}[\hat \D(h,h',S_h\cap S_{h'})]$ and the empirical randomized loss $\E_\rho[\hat L(h,S_h)]$ versus the ratio between the TND bound and the FO bound. The data sets Mushroom, Shuttle and Protein are excluded. The first two because the randomized loss is extremely small. And the third one because the bounds are higher than 1.}
    \label{fig:ratiofDisagreement}
\end{figure} 

\subsection{Standard random forests with optimized weights}\label{apx:bagging-optimized}
This section contains numerical values and additional figures for the optimization experiments provided in the second experiment in the body (Figure~\ref{fig:opt_mvrisk_and_rho}).
$\FO$ was optimized using Theorem~\ref{thm:lambdabound} and the alternating update rules of \citep{TIWS17}. For optimizing $\TND$, we used iRProp$^+$ \citep{igel:01e}, see Appendix~\ref{app:minimization}. We denote the weights after optimization of $\FO$ and $\TND$ by $\rho^*_{\FO}$ and $\rho^*_{\TND}$, respectively.

Figures~\ref{fig:opt_bounds_bin} and \ref{fig:opt_bounds_mul} show the bounds before and after optimization for binary and multiclass data sets respectively. The $\FO$ bound achieves higher reduction after minimization, however, as illustrated in both figures and Figure~\ref{fig:opt_mvrisk_and_rho} in the body, this improvement comes at the cost of considerable increase of the test loss $L(MV_{\rho^*_{\FO}},\testset)$. The latter happens because $\FO$ places most of the posterior mass on a few top classifiers and diminishes the power of the ensemble, see Figure~\ref{fig:example_rho}. The improvement of the $\TND$ after minimization is more modest, but on a highly positive side it does not degrade the classifier. 
\begin{figure}[ht]
    \centering
    \includegraphics[width=\linewidth]{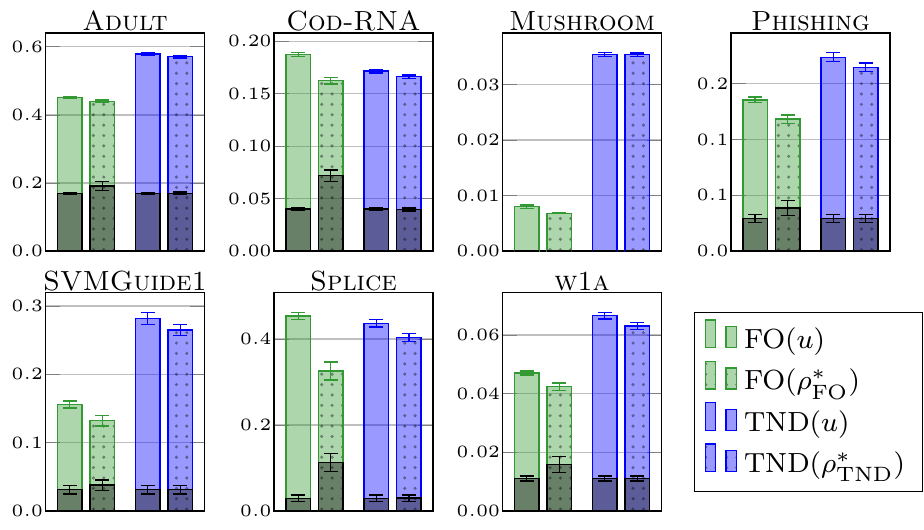}
    \caption{Comparison of the bounds before (not dotted bars) and after (dotted bars) optimization for the binary data sets. The test risk is shown in black.}
    \label{fig:opt_bounds_bin}
\end{figure}
\begin{figure}[ht]
    \centering
    \includegraphics[width=\linewidth]{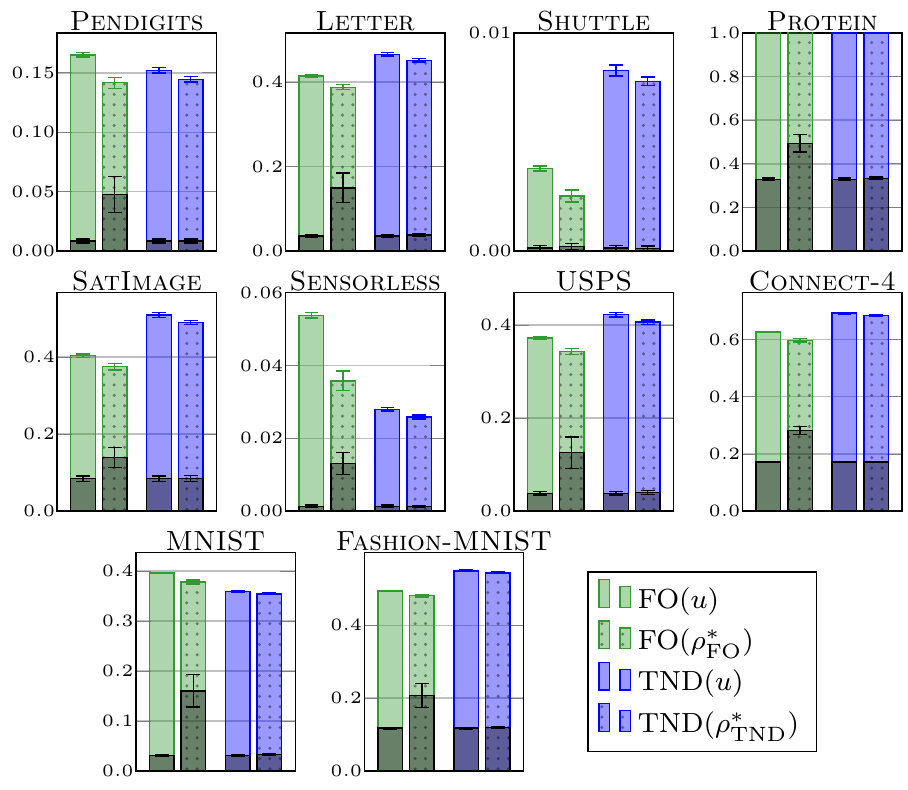}
    \caption{Comparison of the bounds before (not dotted bars) and after (dotted bars) optimization for the multiclass data sets. The test risk is shown in black.}
    \label{fig:opt_bounds_mul}
\end{figure} 

Table~\ref{tab:opt_mvrisk} shows the numerical values used in Figure~\ref{fig:opt_mvrisk}.

\subsection{Random forests with reduced bagging vs.\ full bagging with uniform and optimized weights}\label{apx:rf-reduced}

The $\TND$ bound depends on the size of overlaps $S_h\cap S_{h'}$, which are used to estimate the tandem losses and define the denominator of the bound. In order to ensure that the overlaps $S_h\cap S_{h'}$ are not too small, it might be beneficial to generate splits with $|S_h|$ of at least $(2/3)n$, so that $|S_h\cap S_{h'}|$ is at least $n/3$. In our application to random forests we reduce the number of sampled points in bagging from $n$ to $n/2$, which increases the number of out-of-bag samples $|S_h|$ from roughly $n/3$ to roughly $(2/3)n$ and the overlaps from roughly $n/9$ to $n/3$. We show that the corresponding decrease in $|T_h|$ leads to a relatively small decrease of prediction quality of individual trees and improves the bounds. 

We call the bagging procedure that samples $n$ points with replacement a \emph{standard bagging} or \emph{full bagging} and the procedure that samples $n/2$ points \emph{reduced bagging}. This section presents results for random forests trained with \emph{reduced bagging}, including comparisons to the full bagging setting.

Figure~\ref{fig:mvrisk_compare} compares the test risk in the full bagging and the reduced bagging settings with uniform and optimized weights.
In both uniform and optimized weights we see a limited increase (and in a few cases even a small decrease) in test risk when reducing the amount of data sampled in bagging, indicating that reduced bagging has relatively minor impact on the quality of a uniformly weighted ensemble. At the same time, Figures~\ref{fig:bound_compare}, \ref{fig:bounds_compare_binary}, \ref{fig:bounds_compare_multi}, \ref{fig:opt_bound_reduce_compare_bin}, and \ref{fig:opt_bound_reduce_compare_mul} show that the bounds are improved in most cases, sometimes considerably.

Table~\ref{tab:rf_sample_bounds} reports the means and standard deviations for all data sets. Additional information (randomized loss, tandem loss, etc.) is reported in Table~\ref{tab:rf_reduced_info}. Table~\ref{tab:opt_mvrisk_reduced} reports the performance of the final majority vote with and without optimized weights.


\begin{figure}[ht]
    \centering
    \includegraphics[width=\linewidth]{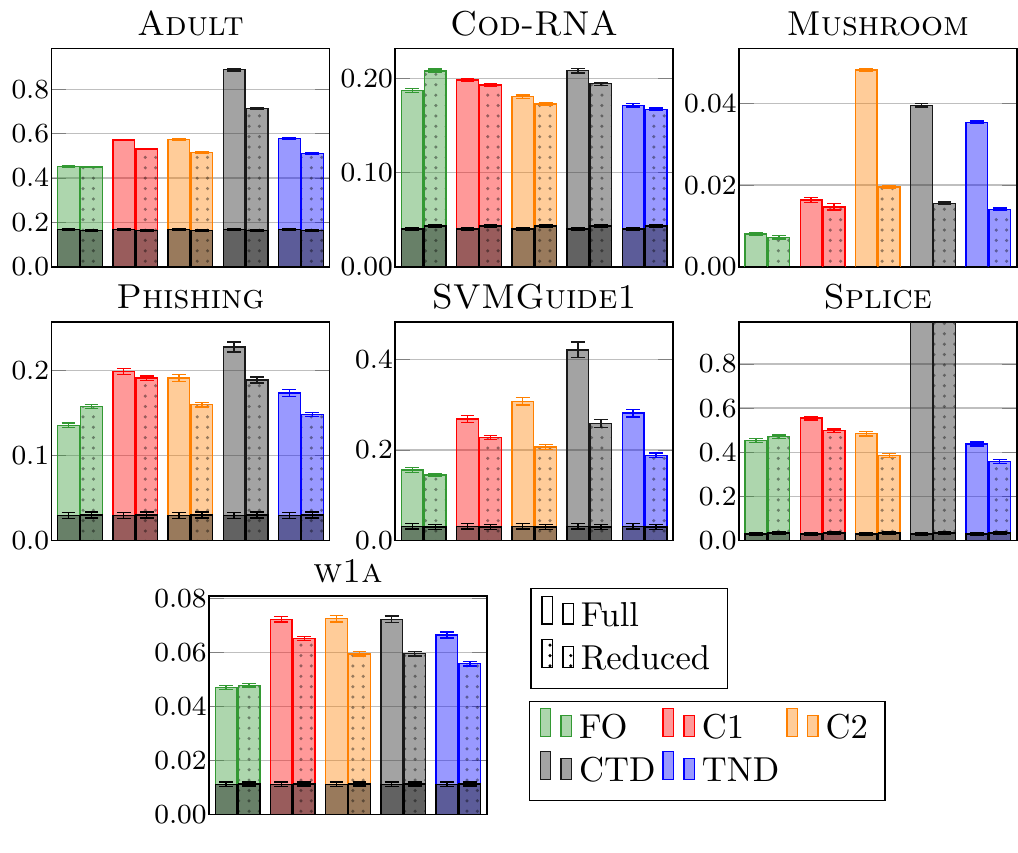}
    \caption{Comparison of the bounds in the full (not dotted) and reduced bagging (dotted) setting with uniform weighting for binary data sets. The test risk is shown in black.
    }
    \label{fig:bounds_compare_binary}
\end{figure}
\begin{figure}[ht]
    \centering
    \includegraphics[width=\linewidth]{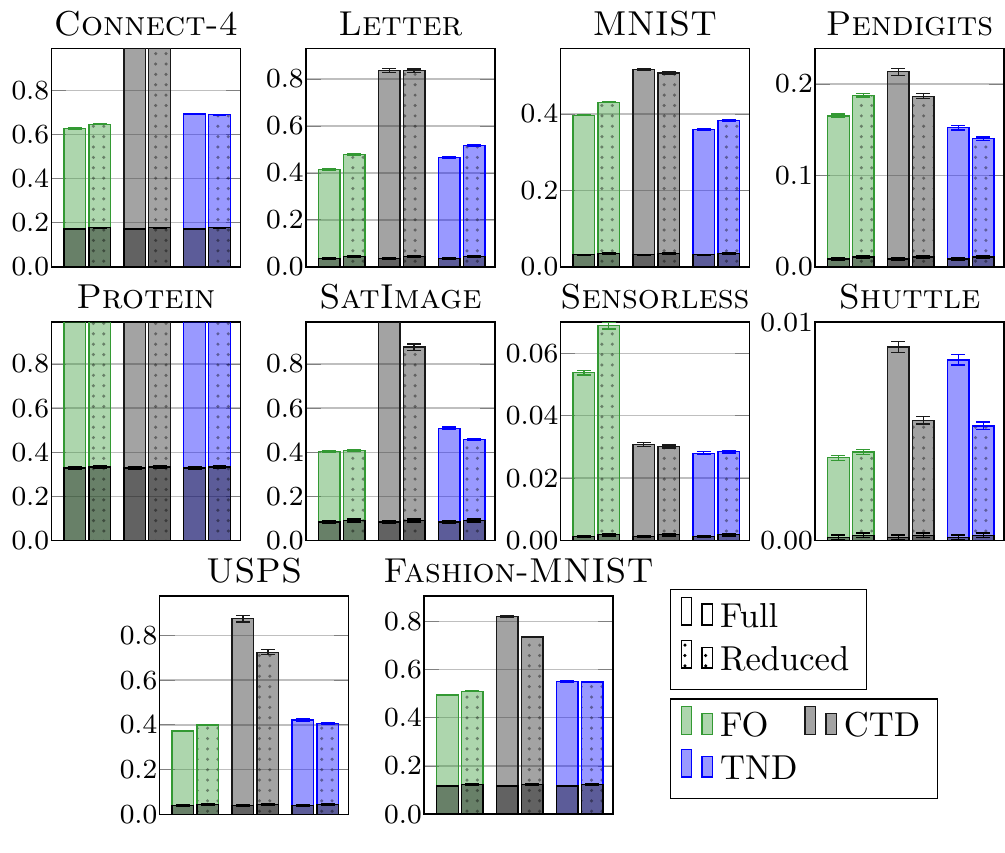}
    \caption{Comparison of the bounds in the full (not dotted) and reduced bagging (dotted) setting with uniform weighting for multiclass data sets. The test risk is shown in black.
    }
    \label{fig:bounds_compare_multi}
\end{figure}
\begin{figure}[ht]
    \centering
    \includegraphics[width=\linewidth]{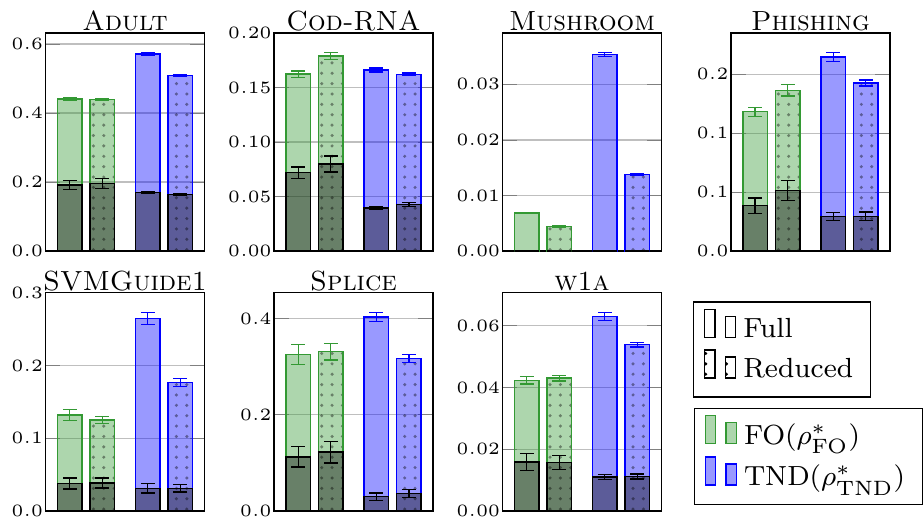}
    \caption{Comparison of the bounds computed for the random forest with optimized weights in the standard bagging (not dotted) and the reduced bagging (dotted) setting on binary data sets.}
    \label{fig:opt_bound_reduce_compare_bin}
\end{figure}
\begin{figure}[ht]
    \centering
    \includegraphics[width=\linewidth]{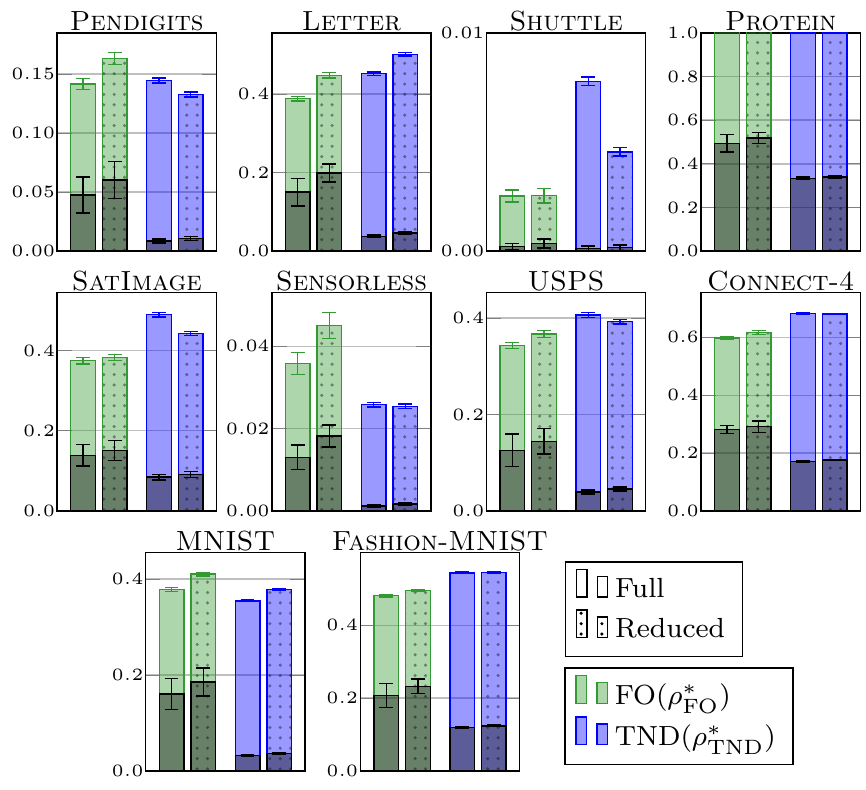}
    \caption{Comparison of the bounds computed for the random forest with optimized weights in the standard bagging (not dotted) and the reduced bagging (dotted) setting on multiclass data sets.}
    \label{fig:opt_bound_reduce_compare_mul}
\end{figure}


\begin{figure}[ht]
    \centering
    \includegraphics{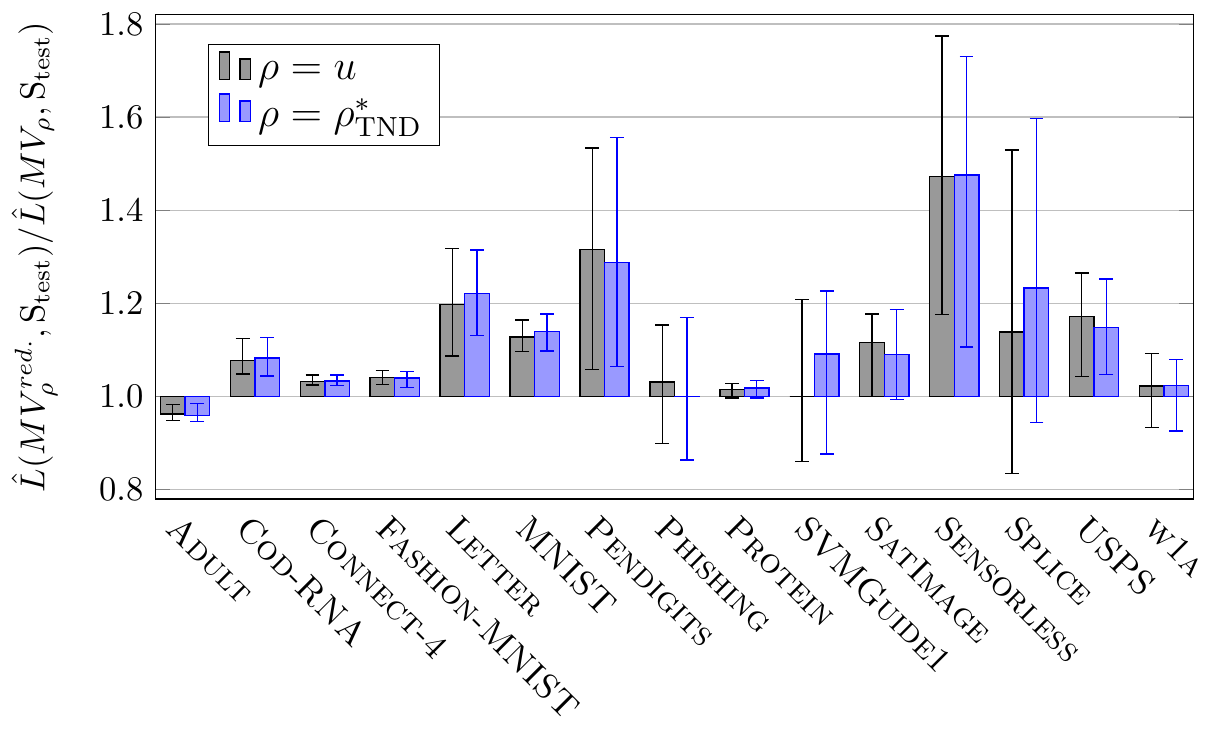}
    \caption{Median, 25\%, and 75\% quantiles of the ratio between the test risk in the reduced and full bagging settings with uniform and optimized weights $\rho^*_{\TND}$. Results on \dataset{Mushroom} and \dataset{Shuttle} are left out, as the test risk is 0 in some cases.}
    \label{fig:mvrisk_compare}
\end{figure}

\begin{figure}[ht]
    \centering
    \includegraphics{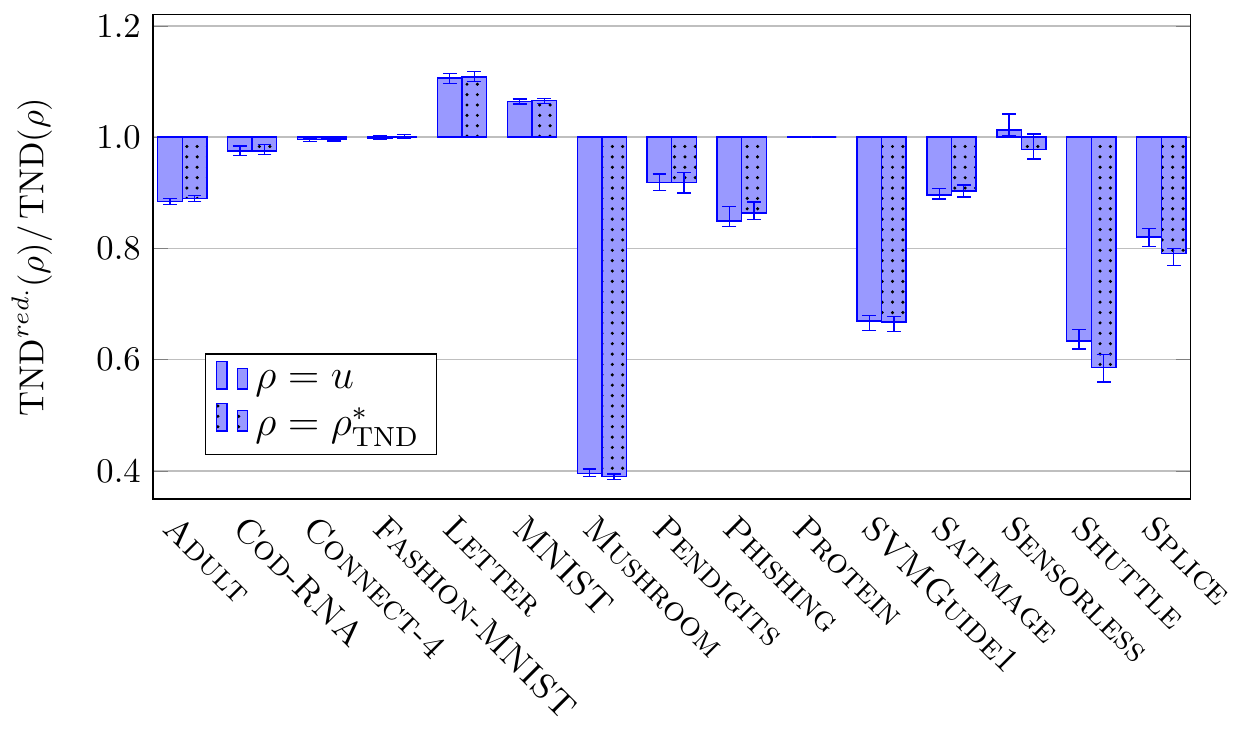}
    \caption{Median, 25\%, and 75\% quantiles of the ratio between the $\TND$ bound in the reduced and full bagging settings with uniform and optimized weights $\rho^*_{\TND}$.}
    \label{fig:bound_compare}
\end{figure}
\subsection{$\DIS$ bound vs.\ $\TND$ bound in presence of unlabeled data}\label{apx:unlabeled}
In this section we compare the tightness of the $\TND$ and $\DIS$ bounds in a setting, where a lot of unlabeled data is available.

We considered the largest binary data sets ($N>8000$) from Table~\ref{tab:data_sets}. As in the previous setting, 20\% of the data, $\testset$, was reserved for testing. The remaining 80\%, were split with a fraction $r\in [0,1]$ of patterns $S$ used for training, and a fraction $(1-r)$ set aside as unlabeled patterns, $S_u$. Forests with 100 trees were trained with bagging, using the Gini criterion for splitting and considering $\sqrt{d}$ features in each split. We considered values of $r\in \{0.05,0.1,...,0.5\}$. For each split, we repeated the experiment 20 times.

\begin{figure}[ht]
    \centering
    \includegraphics[width=\linewidth]{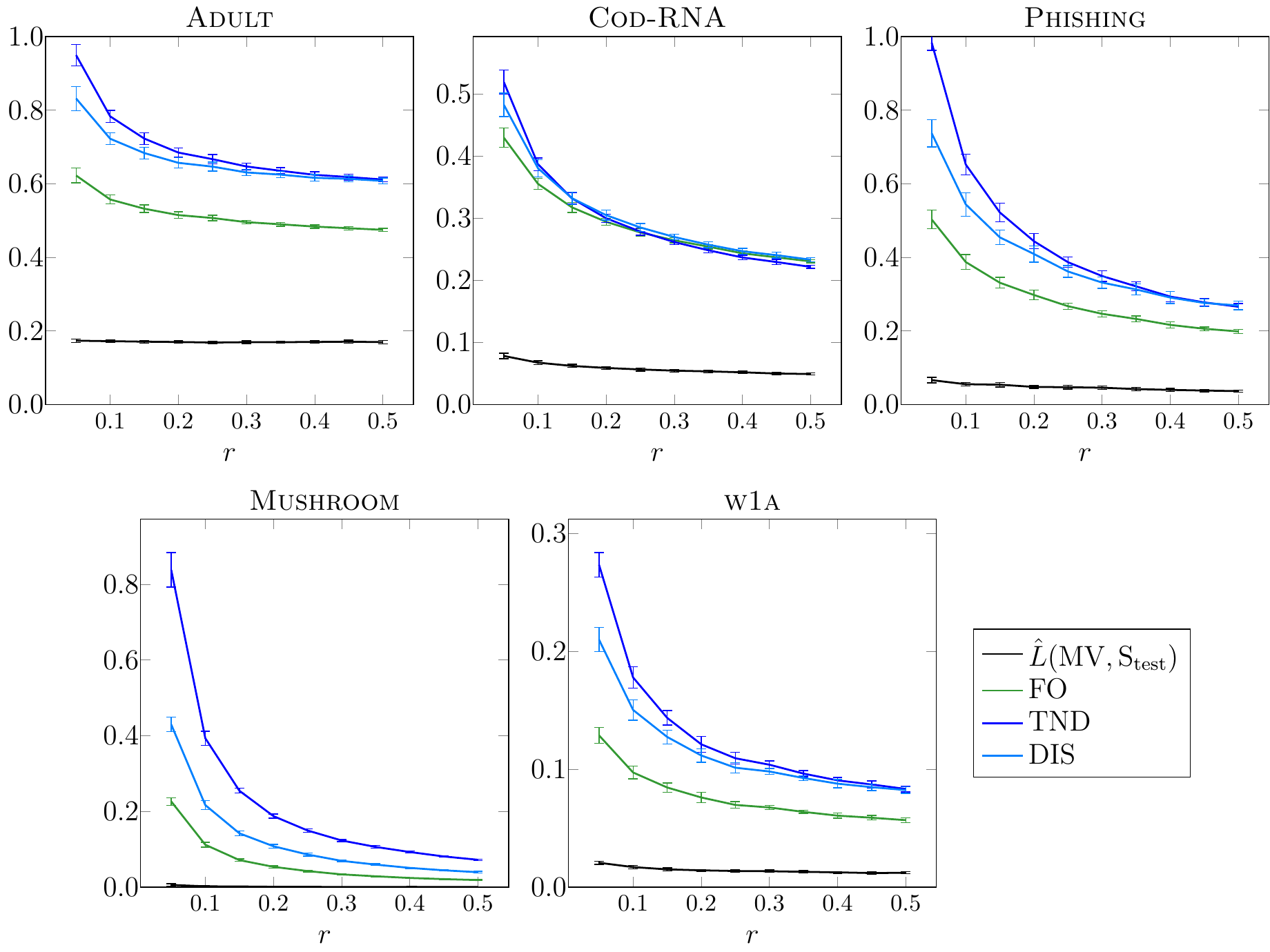}
    \caption{Test risk, $\FO$, $\TND$ and $\DIS$ bounds as a function of the fraction $r$ of labeled points. Means and standard deviations over 20 runs are reported.}
    \label{fig:rf_unlabeled}
\end{figure}
Figure~\ref{fig:rf_unlabeled} plots the test risk and $\FO$, $\TND$ and $\DIS$ bounds as a function of $r$. For each data set, the mean and standard deviation over 20 runs are plotted. In agreement with the discussion in Section~\ref{sec:fast-vs-slow}, $\DIS$ had the highest advantage over $\TND$ when the amount of unlabeled data relative to labeled data was the largest. As the amount of unlabeled data relative to labeled data was decreasing the difference between the bounds became smaller, with $\TND$ eventually overtaking $\DIS$ in most cases.

\begin{landscape}
\begin{table}[ht]
    \caption{Values of the computed bounds when using bagging. Mean and standard deviation are reported for 50 runs on each data set. The tightest bound overall is marked in \textbf{bold}, with the second tightest marked with \underline{underline}. Standard deviations are given in parenthesis. Note that the bounds $\Cone$ and $\Ctwo$ are only defined for binary data sets.}
    \label{tab:rf_bagging_bounds}
    \centering
    \begin{tabular}{@{}lcccccc@{}}\toprule
Dataset & $\hat{L}(\MV,\testset)$ & $\FO$ & $\Cone$ & $\Ctwo$ & $\CTD$ & $\TND$\\
\midrule
\dataset{Adult} & 0.16941 (0.00303) & \textbf{0.45164} (0.00213) & \underline{0.57227} (0.00256) & 0.57318 (0.00312) & 0.88830 (0.00626) & 0.57882 (0.00364) \\
\dataset{Cod-RNA} & 0.04018 (0.00150) & 0.18717 (0.00194) & 0.19835 (0.00171) & \underline{0.18054} (0.00181) & 0.20835 (0.00240) & \textbf{0.17151} (0.00176) \\
\dataset{Connect-4} & 0.17120 (0.00204) & \textbf{0.62706} (0.00148) & - & - & $>1$ & \underline{0.69258} (0.00230) \\
\dataset{Fashion-MNIST} & 0.11752 (0.00228) & \textbf{0.49426} (0.00138) & - & - & 0.81911 (0.00390) & \underline{0.54943} (0.00229) \\
\dataset{Letter} & 0.03602 (0.00315) & \textbf{0.41503} (0.00235) & - & - & 0.83652 (0.00872) & \underline{0.46613} (0.00374) \\
\dataset{MNIST} & 0.03144 (0.00134) & \underline{0.39628} (0.00117) & - & - & 0.51669 (0.00267) & \textbf{0.35937} (0.00154) \\
\dataset{Mushroom} & 0.00000 (0.00000) & \textbf{0.00801} (0.00031) & \underline{0.01638} (0.00061) & 0.04817 (0.00043) & 0.03952 (0.00040) & 0.03539 (0.00034) \\
\dataset{Pendigits} & 0.00854 (0.00183) & \underline{0.16515} (0.00181) & - & - & 0.21330 (0.00380) & \textbf{0.15211} (0.00238) \\
\dataset{Phishing} & 0.02916 (0.00356) & \textbf{0.13554} (0.00255) & 0.19865 (0.00351) & 0.19105 (0.00415) & 0.22756 (0.00597) & \underline{0.17367} (0.00409) \\
\dataset{Protein} & 0.32959 (0.00500) & $>1$ & - & - & $>1$ & $>1$ \\
\dataset{SVMGuide1} & 0.03129 (0.00628) & \textbf{0.15559} (0.00521) & \underline{0.26894} (0.00782) & 0.30776 (0.00823) & 0.42148 (0.01735) & 0.28161 (0.00865) \\
\dataset{SatImage} & 0.08386 (0.00716) & \textbf{0.40328} (0.00409) & - & - & $>1$ & \underline{0.50910} (0.00605) \\
\dataset{Sensorless} & 0.00131 (0.00034) & 0.05380 (0.00077) & - & - & \underline{0.03077} (0.00055) & \textbf{0.02797} (0.00049) \\
\dataset{Shuttle} & 0.00015 (0.00011) & \textbf{0.00379} (0.00011) & - & - & 0.00886 (0.00025) & \underline{0.00827} (0.00024) \\
\dataset{Splice} & 0.02957 (0.00798) & \underline{0.45351} (0.00822) & 0.55494 (0.00837) & 0.48391 (0.00930) & $>1$ & \textbf{0.43683} (0.00903) \\
\dataset{USPS} & 0.03820 (0.00395) & \textbf{0.37237} (0.00289) & - & - & 0.87450 (0.01387) & \underline{0.42214} (0.00463) \\
\dataset{w1a} & 0.01108 (0.00081) & \textbf{0.04704} (0.00065) & 0.07233 (0.00096) & 0.07254 (0.00114) & 0.07234 (0.00122) & \underline{0.06649} (0.00111) \\
\bottomrule
\end{tabular}

\end{table}
\end{landscape}
\begin{table}[ht]
    \caption{Statistics for each data set when using bagging. We use the following short-hand: $\E_\rho[\hat{L}]=\E_\rho[\hat{L}(h,S_h)]$, $\E_{\rho^2}[\hat{\D}]= \E_{\rho^2}[\hat{\D}(h,h',S_h\cap S_{h'})]$, $\E_{\rho^2}[\hat{L}] = \E_{\rho^2}[\hat{L}(h,h',S_h\cap S_{h'})]$}
    \label{tab:rf_bagging_info}
    \centering
    \begin{tabular}{@{}lcccccc@{}}\toprule
Dataset & $\hat{L}(\MV,\testset)$ & $\E_\rho[\hat{L}]$ & $\min |S_h|$ & $\E_{\rho^2}[\hat{\D}]$ & $\E_{\rho^2}[\hat{L}]$ & $\min |S_h\cap S_{h'}|$\\
\midrule
\dataset{Adult} & 0.16941 & 0.20851 & 9459.96 & 0.17422 & 0.12138 & 3359.92 \\
\dataset{Cod-RNA} & 0.04018 & 0.08456 & 17348.12 & 0.10315 & 0.03298 & 6228.60 \\
\dataset{Connect-4} & 0.17120 & 0.29985 & 19697.48 & 0.34343 & 0.15527 & 7077.06 \\
\dataset{Fashion-MNIST} & 0.11752 & 0.23465 & 20410.22 & 0.28396 & 0.12142 & 7335.90 \\
\dataset{Letter} & 0.03602 & 0.18644 & 5785.64 & 0.26280 & 0.08995 & 2032.40 \\
\dataset{MNIST} & 0.03144 & 0.18663 & 20416.56 & 0.27435 & 0.07668 & 7337.68 \\
\dataset{Mushroom} & 0.00000 & 0.00019 & 2327.40 & 0.00036 & 0.00001 & 797.00 \\
\dataset{Pendigits} & 0.00854 & 0.06403 & 3158.70 & 0.10004 & 0.01828 & 1095.48 \\
\dataset{Phishing} & 0.02916 & 0.05097 & 3178.92 & 0.05747 & 0.02224 & 1100.30 \\
\dataset{Protein} & 0.32959 & 0.53934 & 7065.98 & 0.59019 & 0.31349 & 2497.42 \\
\dataset{SVMGuide1} & 0.03129 & 0.04606 & 869.08 & 0.04601 & 0.02297 & 284.54 \\
\dataset{SatImage} & 0.08386 & 0.16636 & 1837.52 & 0.19830 & 0.08061 & 623.54 \\
\dataset{Sensorless} & 0.00131 & 0.02193 & 17049.78 & 0.03845 & 0.00318 & 6112.32 \\
\dataset{Shuttle} & 0.00015 & 0.00069 & 16899.58 & 0.00097 & 0.00023 & 6055.44 \\
\dataset{Splice} & 0.02957 & 0.17561 & 895.60 & 0.25165 & 0.04976 & 293.14 \\
\dataset{USPS} & 0.03820 & 0.15736 & 2668.56 & 0.22115 & 0.06944 & 916.28 \\
\dataset{w1a} & 0.01108 & 0.01852 & 14483.72 & 0.01695 & 0.01005 & 5176.26 \\
\bottomrule
\end{tabular}

\end{table}

\begin{landscape}
\begin{table}[ht]
    \caption{Values of the computed bounds when using reduced bagging. Mean and standard deviation are reported for 50 runs on each data set. The tightest bound overall is marked in \textbf{bold}, with the second tightest marked with \underline{underline}. Standard deviations are given in parenthesis.  Note that the bounds $\Cone$ and $\Ctwo$ are only defined for binary data sets.}
    \label{tab:rf_sample_bounds}
    \centering
    \begin{tabular}{@{}lcccccc@{}}\toprule
Dataset & $\hat{L}(\MV,\testset)$ & $\FO$ & $\Cone$ & $\Ctwo$ & $\CTD$ & $\TND$\\
\midrule
\dataset{Adult} & 0.16370 (0.00322) & \textbf{0.44940} (0.00193) & 0.53068 (0.00226) & 0.51552 (0.00241) & 0.71364 (0.00431) & \underline{0.51148} (0.00258) \\
\dataset{Cod-RNA} & 0.04346 (0.00158) & 0.20793 (0.00194) & 0.19300 (0.00132) & \underline{0.17299} (0.00126) & 0.19433 (0.00161) & \textbf{0.16725} (0.00124) \\
\dataset{Connect-4} & 0.17716 (0.00238) & \textbf{0.64645} (0.00164) & - & - & $>1$ & \underline{0.68942} (0.00215) \\
\dataset{Fashion-MNIST} & 0.12258 (0.00244) & \textbf{0.50935} (0.00138) & - & - & 0.73495 (0.00307) & \underline{0.54870} (0.00204) \\
\dataset{Letter} & 0.04326 (0.00374) & \textbf{0.47741} (0.00234) & - & - & 0.83563 (0.00740) & \underline{0.51579} (0.00402) \\
\dataset{MNIST} & 0.03553 (0.00153) & \underline{0.43055} (0.00140) & - & - & 0.50697 (0.00288) & \textbf{0.38250} (0.00180) \\
\dataset{Mushroom} & 0.00000 (0.00000) & \textbf{0.00724} (0.00039) & 0.01468 (0.00075) & 0.01959 (0.00036) & 0.01554 (0.00035) & \underline{0.01408} (0.00033) \\
\dataset{Pendigits} & 0.01052 (0.00164) & 0.18755 (0.00177) & - & - & \underline{0.18659} (0.00280) & \textbf{0.14001} (0.00187) \\
\dataset{Phishing} & 0.03009 (0.00364) & \underline{0.15753} (0.00222) & 0.19101 (0.00261) & 0.15970 (0.00262) & 0.18902 (0.00357) & \textbf{0.14832} (0.00253) \\
\dataset{Protein} & 0.33413 (0.00552) & $>1$ & - & - & $>1$ & $>1$ \\
\dataset{SVMGuide1} & 0.03006 (0.00541) & \textbf{0.14470} (0.00359) & 0.22802 (0.00481) & 0.20687 (0.00532) & 0.25874 (0.00852) & \underline{0.18814} (0.00529) \\
\dataset{SatImage} & 0.09063 (0.00713) & \textbf{0.40936} (0.00352) & - & - & 0.87653 (0.01369) & \underline{0.45752} (0.00490) \\
\dataset{Sensorless} & 0.00187 (0.00038) & 0.06891 (0.00111) & - & - & \underline{0.03011} (0.00048) & \textbf{0.02847} (0.00045) \\
\dataset{Shuttle} & 0.00024 (0.00011) & \textbf{0.00405} (0.00010) & - & - & 0.00551 (0.00017) & \underline{0.00525} (0.00017) \\
\dataset{Splice} & 0.03398 (0.00759) & 0.47039 (0.00885) & 0.49763 (0.00868) & \underline{0.38612} (0.00884) & $>1$ & \textbf{0.35799} (0.00857) \\
\dataset{USPS} & 0.04385 (0.00457) & \textbf{0.39916} (0.00328) & - & - & 0.72489 (0.01091) & \underline{0.40572} (0.00434) \\
\dataset{w1a} & 0.01121 (0.00068) & \textbf{0.04782} (0.00064) & 0.06528 (0.00075) & 0.05949 (0.00079) & 0.05955 (0.00083) & \underline{0.05585} (0.00077) \\
\bottomrule
\end{tabular}

\end{table}
\end{landscape}
\begin{table}[ht]
    \caption{Statistics for each data set when using reduced bagging. We use the following short-hand: $\E_\rho[\hat{L}]=\E_\rho[\hat{L}(h,S_h)]$, $\E_{\rho^2}[\hat{\D}]= \E_{\rho^2}[\hat{\D}(h,h',S_h\cap S_{h'})]$, $\E_{\rho^2}[\hat{L}] = \E_{\rho^2}[\hat{L}(h,h',S_h\cap S_{h'})]$}
    \label{tab:rf_reduced_info}
    \centering
    \begin{tabular}{@{}lcccccc@{}}\toprule
Dataset & $\hat{L}(\MV,\testset)$ & $\E_\rho[\hat{L}]$ & $\min |S_h|$ & $\E_{\rho^2}[\hat{\D}]$ & $\E_{\rho^2}[\hat{L}]$ & $\min |S_h\cap S_{h'}|$\\
\midrule
\dataset{Adult} & 0.16370 & 0.21105 & 15702.72 & 0.19390 & 0.11409 & 9401.66 \\
\dataset{Cod-RNA} & 0.04346 & 0.09649 & 28763.86 & 0.12166 & 0.03566 & 17277.52 \\
\dataset{Connect-4} & 0.17716 & 0.31235 & 32643.02 & 0.35965 & 0.16125 & 19614.80 \\
\dataset{Fashion-MNIST} & 0.12258 & 0.24472 & 33827.88 & 0.29624 & 0.12725 & 20337.04 \\
\dataset{Letter} & 0.04326 & 0.22119 & 9630.06 & 0.30819 & 0.11160 & 5745.70 \\
\dataset{MNIST} & 0.03553 & 0.20590 & 33828.08 & 0.30029 & 0.08716 & 20331.38 \\
\dataset{Mushroom} & 0.00000 & 0.00057 & 3894.98 & 0.00104 & 0.00005 & 2299.78 \\
\dataset{Pendigits} & 0.01052 & 0.07817 & 5276.38 & 0.12164 & 0.02289 & 3127.52 \\
\dataset{Phishing} & 0.03009 & 0.06443 & 5308.52 & 0.07962 & 0.02464 & 3146.54 \\
\dataset{Protein} & 0.33413 & 0.54520 & 11746.54 & 0.59434 & 0.31870 & 7019.70 \\
\dataset{SVMGuide1} & 0.03006 & 0.04796 & 1470.06 & 0.05098 & 0.02247 & 853.08 \\
\dataset{SatImage} & 0.09063 & 0.17665 & 3080.48 & 0.21070 & 0.08664 & 1813.18 \\
\dataset{Sensorless} & 0.00187 & 0.03000 & 28265.20 & 0.05216 & 0.00462 & 16971.34 \\
\dataset{Shuttle} & 0.00024 & 0.00101 & 28014.66 & 0.00144 & 0.00034 & 16823.72 \\
\dataset{Splice} & 0.03398 & 0.19427 & 1511.18 & 0.27746 & 0.05557 & 876.88 \\
\dataset{USPS} & 0.04385 & 0.17617 & 4458.86 & 0.24642 & 0.07926 & 2636.72 \\
\dataset{w1a} & 0.01121 & 0.01991 & 24022.24 & 0.01957 & 0.01013 & 14412.94 \\
\bottomrule
\end{tabular}

\end{table}

\begin{table}[ht]
    \centering
    \begin{tabular}{@{}lccc@{}}\toprule
Dataset & $\hat L(\MV_u,\testset)$ & $\hat L(\MV_{\rho_{\FO}^*}, \testset)$ & $\hat L(\MV_{\rho_{\TND}^*}, \testset)$\\
\midrule
\dataset{Adult} & \textbf{0.16941} (0.00303) & 0.19136 (0.01335) & \underline{0.17004} (0.00313) \\
\dataset{Cod-RNA} & \underline{0.04018} (0.00150) & 0.07193 (0.00530) & \textbf{0.03963} (0.00138) \\
\dataset{Connect-4} & \textbf{0.17120} (0.00204) & 0.28148 (0.01407) & \underline{0.17123} (0.00202) \\
\dataset{Fashion-MNIST} & \textbf{0.11752} (0.00228) & 0.20678 (0.03283) & \underline{0.11895} (0.00222) \\
\dataset{Letter} & \textbf{0.03602} (0.00315) & 0.14998 (0.03493) & \underline{0.03784} (0.00336) \\
\dataset{MNIST} & \textbf{0.03144} (0.00134) & 0.16014 (0.03238) & \underline{0.03223} (0.00137) \\
\dataset{Mushroom} & \textbf{0.00000} (0.00000) & \textbf{0.00000} (0.00000) & \textbf{0.00000} (0.00000) \\
\dataset{Pendigits} & \textbf{0.00854} (0.00183) & 0.04752 (0.01515) & \underline{0.00856} (0.00168) \\
\dataset{Phishing} & \textbf{0.02916} (0.00356) & 0.03865 (0.00649) & \underline{0.02935} (0.00355) \\
\dataset{Protein} & \textbf{0.32959} (0.00500) & 0.49377 (0.03958) & \underline{0.33402} (0.00578) \\
\dataset{SVMGuide1} & \underline{0.03129} (0.00628) & 0.03786 (0.00764) & \textbf{0.03120} (0.00637) \\
\dataset{SatImage} & \textbf{0.08386} (0.00716) & 0.13876 (0.02631) & \underline{0.08437} (0.00711) \\
\dataset{Sensorless} & \underline{0.00131} (0.00034) & 0.01304 (0.00298) & \textbf{0.00118} (0.00029) \\
\dataset{Shuttle} & \underline{0.00015} (0.00011) & 0.00022 (0.00015) & \textbf{0.00013} (0.00011) \\
\dataset{Splice} & \textbf{0.02957} (0.00798) & 0.11257 (0.02121) & \underline{0.03005} (0.00769) \\
\dataset{USPS} & \textbf{0.03820} (0.00395) & 0.12554 (0.03381) & \underline{0.03954} (0.00417) \\
\dataset{w1a} & \underline{0.01108} (0.00081) & 0.01586 (0.00279) & \textbf{0.01106} (0.00081) \\
\bottomrule
\end{tabular}

    \caption{Test risks computed when using different bounds for optimizing $\rho$. Best risk achieved overall is marked in \textbf{bold}, while best risk achieved by optimization is marked with \underline{underline}.  
    }
    \label{tab:opt_mvrisk}
\end{table}

\begin{table}[ht]
    \centering
    \begin{tabular}{@{}lccc@{}}\toprule
Dataset & $\hat L(\MV_u,\testset)$ & $\hat L(\MV_{\rho_{\FO}^*}, \testset)$ & $\hat L(\MV_{\rho_{\TND}^*}, \testset)$\\
\midrule
\dataset{Adult} & \textbf{0.16370} (0.00322) & 0.19592 (0.01385) & \underline{0.16427} (0.00337) \\
\dataset{Cod-RNA} & \underline{0.04346} (0.00158) & 0.07990 (0.00725) & \textbf{0.04282} (0.00167) \\
\dataset{Connect-4} & \underline{0.17716} (0.00238) & 0.29161 (0.01928) & \textbf{0.17698} (0.00211) \\
\dataset{Fashion-MNIST} & \textbf{0.12258} (0.00244) & 0.23242 (0.01962) & \underline{0.12367} (0.00262) \\
\dataset{Letter} & \textbf{0.04326} (0.00374) & 0.19865 (0.02292) & \underline{0.04613} (0.00341) \\
\dataset{MNIST} & \textbf{0.03553} (0.00153) & 0.18514 (0.02914) & \underline{0.03662} (0.00164) \\
\dataset{Mushroom} & \textbf{0.00000} (0.00000) & \textbf{0.00000} (0.00000) & \textbf{0.00000} (0.00000) \\
\dataset{Pendigits} & \textbf{0.01052} (0.00164) & 0.06012 (0.01572) & \underline{0.01070} (0.00174) \\
\dataset{Phishing} & \underline{0.03009} (0.00364) & 0.05129 (0.00849) & \textbf{0.02958} (0.00370) \\
\dataset{Protein} & \textbf{0.33413} (0.00552) & 0.51822 (0.02526) & \underline{0.33895} (0.00504) \\
\dataset{SVMGuide1} & \textbf{0.03006} (0.00541) & 0.03845 (0.00701) & \underline{0.03126} (0.00532) \\
\dataset{SatImage} & \textbf{0.09063} (0.00713) & 0.15094 (0.02514) & \underline{0.09114} (0.00690) \\
\dataset{Sensorless} & \underline{0.00187} (0.00038) & 0.01819 (0.00269) & \textbf{0.00171} (0.00040) \\
\dataset{Shuttle} & \underline{0.00024} (0.00011) & 0.00035 (0.00020) & \textbf{0.00016} (0.00012) \\
\dataset{Splice} & \textbf{0.03398} (0.00759) & 0.12252 (0.02238) & \underline{0.03657} (0.00857) \\
\dataset{USPS} & \textbf{0.04385} (0.00457) & 0.14450 (0.02630) & \underline{0.04534} (0.00418) \\
\dataset{w1a} & \underline{0.01121} (0.00068) & 0.01572 (0.00234) & \textbf{0.01117} (0.00078) \\
\bottomrule
\end{tabular}

    \caption{Test risks computed when using different bounds for optimizing $\rho$ for random forest trained using reduced bagging. Best risk achieved overall is marked in \textbf{bold}, while best risk achieved by optimization is marked with \underline{underline}.
    }
    \label{tab:opt_mvrisk_reduced}
\end{table}
\end{document}